\documentclass[pdflatex,sn-mathphys-num]{sn-jnl}

\usepackage{graphicx}%
\usepackage{multirow}%
\usepackage{amsmath,amssymb,amsfonts}%
\usepackage{amsthm}%
\usepackage{mathrsfs}%
\usepackage[title]{appendix}%
\usepackage{xcolor}%
\usepackage{textcomp}%
\usepackage{manyfoot}%
\usepackage{booktabs}%
\usepackage{algorithm}%
\usepackage{algorithmicx}%
\usepackage{algpseudocode}%
\usepackage{listings}%

\usepackage{mathtools}%

\usepackage[capitalize,nameinlink]{cleveref}
\Crefname{equation}{}{}
\Crefname{figure}{Figure}{}

\newcommand\blfootnote[1]{%
  \begingroup
  \renewcommand\thefootnote{}%
  \footnotetext{#1}%
  \addtocounter{footnote}{-1}%
  \endgroup
}

\theoremstyle{thmstyleone}%
\newtheorem{theorem}{Theorem}[section]%
\newtheorem{lemma}[theorem]{Lemma}
\newtheorem{prop}[theorem]{Proposition}

\theoremstyle{thmstyletwo}%
\newtheorem{example}[theorem]{Example}
\newtheorem{remark}[theorem]{Remark}

\theoremstyle{thmstylethree}%
\newtheorem{ass}[theorem]{Assumption}
\newtheorem{problem}[theorem]{Problem}

\numberwithin{equation}{section}

\AddToHook{env/lemma/begin}{\crefalias{theorem}{lemma}}
\AddToHook{env/prop/begin}{\crefalias{theorem}{prop}}
\AddToHook{env/corollary/begin}{\crefalias{theorem}{corollary}}
\AddToHook{env/problem/begin}{\crefalias{theorem}{problem}}
\AddToHook{env/ass/begin}{\crefalias{theorem}{ass}}
\AddToHook{env/example/begin}{\crefalias{theorem}{example}}
\AddToHook{env/definition/begin}{\crefalias{theorem}{definition}}

\newcommand{\R}{\mathbb{R}}
\newcommand{\N}{\mathbb{N}}
\DeclareMathOperator{\T}{\mathrm{T}}
\DeclareMathOperator{\SA}{\mathrm{SA}}
\DeclareMathOperator{\FF}{\mathrm{FF}}
\DeclareMathOperator{\co}{\overline{co}}
\DeclareMathOperator{\ext}{ext}
\newcommand{\ip}[2]{\langle #1, #2 \rangle}

\newcommand{\C}{\mathcal{C}}

\newcommand{\abs}[1]{\left\vert #1 \right\vert}
\DeclareMathOperator*{\argmax}{arg\,max}
\DeclareMathOperator*{\argmin}{arg\,min}
\DeclareMathOperator{\len}{len}

\begin{document}

\title[Exact Sequence Interpolation with Transformers]{Exact Sequence Interpolation with Transformers}


\author*[1]{\fnm{Albert} \sur{Alcalde}}\email{albert.alcalde@fau.de}

\author[1]{\fnm{Giovanni} \sur{Fantuzzi}}\email{giovanni.fantuzzi@fau.de}

\author[1,2,3]{\fnm{Enrique} \sur{Zuazua}}\email{enrique.zuazua@fau.de}

\affil*[1]{\orgdiv{Chair for Dynamics, Control, Machine Learning, and Numerics (Alexander von Humboldt Professorship)}, \orgname{Friedrich--Alexander--Universit\"at Erlangen--N\"urnberg}, \orgaddress{\street{Cauerstrasse, 11}, \city{Erlangen}, \postcode{91058}, \country{Germany}}}

\affil[2]{\orgdiv{Departamento de Matem\'{a}ticas}, \orgname{Universidad Aut\'{o}noma de Madrid}, \city{Madrid}, \postcode{28049}, \country{Spain}}

\affil[3]{\orgdiv{Chair of Computational Mathematics}, \orgname{Fundaci\'{o}n Deusto}, \orgaddress{\street{Av. de las Universidades, 24}, \city{Bilbao}, \postcode{48007}, \country{Spain}}}


\abstract{
We prove that transformers can exactly interpolate datasets of finite input sequences in $\R^d$, $d\geq 2$, with corresponding output sequences of smaller or equal length. Specifically, given $N$ sequences of arbitrary but finite lengths in $\R^d$ and output sequences of lengths $m^1, \dots, m^N \in \N$, we construct a transformer with $\mathcal{O}(\sum_{j=1}^N m^j)$ blocks and $\smash{\mathcal{O}(d \sum_{j=1}^N m^j)}$ parameters that exactly interpolates the dataset.
Our construction provides complexity estimates that are independent of the input sequence length, by alternating feed-forward and self-attention layers and by capitalizing on the clustering effect inherent to the latter. Our novel constructive method also uses low-rank parameter matrices in the self-attention mechanism, a common feature of practical transformer implementations. 
These results are first established in the hardmax self-attention setting, where the geometric structure permits an explicit and quantitative analysis, and are then extended to the softmax setting. Finally, we demonstrate the applicability of our exact interpolation construction to learning problems, in particular by providing convergence guarantees to a global minimizer under regularized training strategies. Our analysis contributes to the theoretical understanding of transformer models, offering an explanation for their excellent performance in exact sequence-to-sequence interpolation tasks.
}

\keywords{Deep Learning, Self-Attention, Approximation, Transformer Architecture, Mathematics of Machine Learning}

\maketitle

\blfootnote{Author accepted manuscript. The version of record is published in \emph{Math. Found. Mach. Learn.} 2, 2 (2026), \href{https://doi.org/10.1007/s44439-026-00005-y}{doi:10.1007/s44439-026-00005-y}.}

\section{Introduction}\label{sec:introduction}

Transformers \cite{vaswaniAttentionAllYou2017} have revolutionized machine learning (ML) by outperforming traditional residual networks (ResNets) in applications such as natural language processing \cite{openai2024gpt4technicalreport} and computer vision \cite{dosovitskiy2021imageworth16x16words}, where inputs naturally take the form of finite sequences of $d$-dimensional vectors. 
Their practical success relies on the effectiveness of the self-attention layers, which act between feed-forward layers to update each vector in the sequence by aggregating information from all other vectors.

The expressive power of feed-forward layers has already been studied extensively, see, e.g., \cite{Cybenko1989ApproximationBS, HORNIK1991251, leshno1993multilayer, Pinkus_1999, yarotsky2017error, park2021provable, vardi2022on, li2022deep, domenec2023NODES, alvarez2025clusterClassification}. In particular, it is known \cite{yun2019small} that sufficiently deep feed-forward networks can interpolate arbitrary input data. Unlike feed-forward networks, transformers rely on self-attention, and thus their interpolation properties and complexity scaling require a distinct analysis.

\subsection{Summary of contributions}\label{ss:discussion}

In this work, we prove that certain transformer architectures can perfectly interpolate data and we provide explicit complexity estimates in terms of layers and parameter count. The key insight of our analysis is that transformers can attain the same expressive power as architectures purely based on feed-forward layers with significantly fewer parameters. In more detail, our core contributions are the following.

\begin{itemize}

\item \textit{Exact interpolation with real vector-valued outputs.}
We show via an explicit construction that transformer parameters can always be chosen to interpolate arbitrary real vector-valued sequence-to-sequence maps exactly. We therefore overcome previous limitations to discrete outputs \cite{kim2023provable} or approximate interpolation \cite{geshkovski2024measure} (see also \Cref{ss:relatedWork} below for further discussion).

\item \textit{Complexity independent of input lengths.}
The required transformer size depends only on the total output length, not on the (possibly much longer) input lengths. This helps explain why transformers excel in tasks where long inputs must be mapped to shorter outputs, such as text summarization or classification.

\item \textit{Parameter efficiency through low-rank attention.}
By parameterizing the attention layer of our transformer using matrices that are sums of a rank-one matrix and a multiple of the identity, the number of parameters in our construction grows linearly with the input embedding dimension $d$. This is consistent with efficient transformer designs used in practice.

\item \textit{Implications for regularized training.}
The existence of exact interpolating transformers ensures that global minimizers of regularized training objectives scale linearly with the regularization parameter, giving a practical criterion to detect whether training converges globally or only to local minima.
\end{itemize}

We present the problem statement and our results more precisely in \Cref{ss:prob.statement} and \Cref{ss:mainResults}. Similarities and differences with existing works are discussed in \Cref{ss:relatedWork}. We conclude by comparing the parameter complexity of transformers and ResNets in \Cref{ss:paramComplexity}. 

\subsection{Problem statement}\label{ss:prob.statement}
Fix $d\in \N$. 
For $n \in \mathbb{N}$, we denote by $(\mathbb{R}^d)^n$ the set of tuples of vectors in $\mathbb{R}^d$.
In keeping with standard terminology, we call $X = (x_1, \ldots, x_n) \in (\mathbb{R}^d)^n$ a \emph{sequence} and each vector $x_i$ a \emph{token}. 
Since our constructions are permutation-equivariant \cite{yunAreTransformersUniversal2020}, we identify two sequences if they differ only by a permutation of their tokens, so sequences may be regarded mathematically as unordered sets.
The set of all finite-length non-empty sequences in $\mathbb{R}^d$ is

\begin{equation*}
    \mathcal{X} = \bigcup_{n \in \N} (\R^d)^n,
\end{equation*}
which we view as a topological space equipped with the Hausdorff distance between two sequences $X,Y \in \mathcal{X}$
\begin{equation}\label{eq:hausdorffDistance}
    d_H (X,Y) \coloneqq \max \left( \max_{x\in X} \min_{y\in Y} \| x - y\|_2, \max_{y\in Y} \min_{x\in X} \| x - y\|_2 \right),
\end{equation}
where $\| \cdot \|_2$ denotes the $\ell^2$ norm in $\R^d$. Let $N\in \N$. We consider a dataset of $N$ input-output pairs of sequences: inputs $X^1,\ldots,X^N \in \mathcal{X}$ of lengths $n^1, \dots, n^N \in \N$, and corresponding outputs $Y^1, \dots, Y^N \in \mathcal{X}$ of lengths $m^1, \dots, m^N \in \N$, where $m^j \leq n^j$ for all $j\in [N]\coloneqq \{ 1,\dots, N \}$.

Sequential datasets with the previously described characteristics appear naturally in language and vision applications. In \emph{masked language modeling} \cite{devlin-etal-2019-bert}, each input sequence $X^j$ corresponds to an incomplete sentence and the output $Y^j$ consists of the missing (or ``masked'') words. Analogously, in \emph{masked image modeling} \cite{he2022masked, bao2022beit}, the input sequence $X^j$ represents an image with missing patches, while the output $Y^j$ contains the missing content. For further illustration, we consider the following example in language modeling. 

\begin{example}
Consider the following masked sentences:
\begin{align*}
    X^1 &= \text{``The [mask1] is blue''}, & Y^1 &= \text{``sky''}, \\
    X^2 &= \text{``She likes eating [mask1] [mask2]''}, & Y^2 &= \text{``ice cream''}.
\end{align*}
The inputs $X^1, X^2$ have different lengths ($n^1 = 4$, $n^2 = 5$), and the outputs $Y^1, Y^2$ also differ in length ($m^1 = 1$ and $m^2 = 2$). This illustrates a key feature of our setting: input and output sequences may have different lengths, with the output typically shorter than the input.
\end{example}
The interpolation problem for the dataset consists of constructing a function that maps each input $X^j$ to the corresponding output $Y^j$. 
Transformers are an appropriate class of models to solve such problems, since they process variable-length inputs in a permutation-equivariant way. For our purposes, we view a transformer as a parametrized map $\T : \mathcal{X} \to \mathcal{X}$ acting on sequences of vectors in $\R^d$. Its precise architecture will be described in \Cref{sec:tfArchitecture}, but for the present discussion it suffices to note that transformers preserve sequence length: if $X^j$ has length $n^j$, then $\T(X^j)$ also has length $n^j$. This creates a mismatch with our outputs $Y^j$, which are of equal or shorter length ($m^j \leq n^j$). We resolve this mismatch by viewing equality between sequences in the set sense, that is, two sequences are equal when their Hausdorff distance is equal to zero. We are therefore led to the following exact interpolation problem. 

\begin{problem}[Exact interpolation]\label{pb:exactInterp}
    Construct a transformer $\T : \mathcal{X} \to \mathcal{X}$ such that $ \T(X^j) = Y^j$ for all $j \in [N]$, with explicit bounds on the parameter count.
\end{problem}

\begin{remark}
    \Cref{pb:exactInterp} is formulated for permutation-equivariant transformers, where the order of tokens does not matter. Positional encodings \cite{vaswaniAttentionAllYou2017} can be employed in a separate pre-processing step to add information about the order of tokens if desired, but we do not account for it here because this does not affect the behavior of the transformer. For sequence modeling tasks that require auto-regressive next-token prediction, the permutation equivariance can also be removed by employing ``masked'' self-attention. The relation between this masked auto-regressive setting and \Cref{pb:exactInterp} is discussed in \Cref{sss:maskedSA}, where we show how the former can be reduced to an instance of the latter.
\end{remark}
\begin{remark}
    Viewing the equality $\T(X^j)=Y^j$ as a set equality means that, if the output sequence consists of a single token, then \emph{every} token of the sequence $\T(X^j)$ must coincide with the desired output token. This is harder to achieve compared to matching only a specific token of the input sequence to the desired output, but can be done with a transformer with the same leading-order complexity.
\end{remark}
To ensure \cref{pb:exactInterp} is solvable, we make the following assumptions on the dataset.

\begin{ass}\label{ass:dataset}
    The input and output sequences $\{(X^j, Y^j)\}_{j\in [N]}$ have the following properties:
    \begin{enumerate}
    \item[i)] The sequences $X^1,\ldots, X^N$ are pairwise distinct.
    \item[ii)] Tokens within each $X^j$ are pairwise distinct.
    \item[iii)] Tokens within each $Y^j$ are pairwise distinct.
    \end{enumerate}
\end{ass}

Assumption \textit{i)} could be replaced with the weaker requirement that $X^j=X^{j'}$ implies $Y^j = Y^{j'}$. This condition is necessary to ensure that \Cref{pb:exactInterp} is solvable. However, if $X^{j}=X^{j'}$ and $Y^j = Y^{j'}$, then the pair $(X^{j'}, Y^{j'})$ can be dropped from the dataset without changing \Cref{pb:exactInterp} to ensure that Assumption \textit{i)} holds.
Assumptions \textit{ii)} and \textit{iii)} are mild since positional encoding \cite{press2022trainshorttestlong} usually ensures that tokens in a sequence are distinct. 
The pairwise-distinctness requirement is not essential: with minor proof edits (at the cost of heavier notation), repeated tokens in an input sequence $X^j$ can be replaced by a single token with more ``mass''. Concretely, replace the constant $1/|\mathcal{C}_i(X,A)|$ in \Cref{eq:hardmaxFormulation} with $w_\ell/\sum_{k \in \mathcal{C}_i(X,A)} w_k$, where $w_\ell$ counts the number of times token $x_\ell$ is repeated.

\subsection{Main results}\label{ss:mainResults}
We solve \Cref{pb:exactInterp} for datasets satisfying  \Cref{ass:dataset} in dimension $d\geq 2$\footnote{The case $d=1$ is of limited practical interest, so the restriction is mild.} using a transformer with $\smash{\mathcal{O}(d \sum_j m_j)}$ parameters that alternates single-head self-attention and feed-forward layers with ReLU activations and trainable residual connections. We refer to a single pair of feed-forward and self-attention layers as a transformer \textit{block}. Details of the architecture are described in \Cref{sec:tfArchitecture}.

We analyze two cases, depending on whether the self-attention layers are defined using a hardmax or a softmax attention mechanism. While only the latter is used in practice, the former allows for a simpler geometric intuition that facilitates analysis~\cite{alcalde2025clustering}. Additionally, ``hardmax transformers'' are closely related to ``softmax transformers'' because the hardmax function arises when the so-called temperature parameter in the softmax function tends to zero~\cite{elfadelSoftmax1993}.

Our first main result establishes that hardmax transformers can exactly interpolate data and solve \Cref{pb:exactInterp}.

\begin{theorem}[Exact interpolation with hardmax transformers]
\label{thm:mainResult}
    Fix $N\in \N$, $d\geq 2$ and a dataset of sequences $\{(X^j, Y^j)\}_{j\in [N]}$ satisfying \Cref{ass:dataset}. 
    There exists a hardmax transformer $\T^0$ with feed-forward layers of width $d' \leq 4$, 
    \begin{equation}\label{eq:mainthm}
        L = 2\sum_{j=1}^N m^j + 2N +1\; \text{blocks,} \quad \text{and} \quad 
        P = \mathcal{O}\left(d \sum_{j=1}^N m^j\right)\; \text{parameters}
    \end{equation} 
    such that 
    $$
    \T^0 (X^j) = Y^j \quad \forall j \in [N].
    $$
\end{theorem}

\begin{remark}
   The number of parameters in \Cref{thm:mainResult} agrees up to a prefactor with the lower bound of $\smash{P \geq d \sum_j m^j}$ for any function $\T(\cdot,\theta):\mathcal{X}\to\mathcal{X}$ parameterized in a \emph{smooth} way by $\theta \in \R^P$. Indeed, if a parameter vector $\theta$ exists for every choice of output sequences $\{Y^j\}$ in a $\smash{d\sum_j m^j}$-dimensional subspace $\mathcal{Y}$ of finite-length sequences, then the map from the parameter space $\R^P$ to this space must be surjective. This requires $\smash{P \geq d \sum_j m^j}$, otherwise every point in $\R^P$ would be critical and the image $\T(\cdot,\R^P) \subset \mathcal{Y}$ would have measure zero by Sard's theorem \cite[\S7]{Guillemin1974}.
\end{remark}

The proof of \Cref{thm:mainResult}, presented in \Cref{sec:proofOfMainResult}, relies on a number of technical results proven in \cref{sec:technicalResults} that may be of independent interest. In particular, \Cref{lem:partialConc} shows how a self-attention layer can collapse tokens that lie in a suitable geometric configuration. Conversely, \Cref{lem:chooseLeader,lem:hatFFClassifier} show how self-attention and feed-forward layers can be used to separate overlapping tokens from different sequences. Finally, \Cref{lem:hatFFClassifier} demonstrates that feed-forward layers can move tokens individually once they are separated and collapsed. Together, these results provide simple geometric tools that underlie the construction of the interpolating transformer.

The second main result of this paper, proven in \Cref{sec:proofOfMainResultSoftmax}, is a solution to the exact interpolation problem that uses softmax transformers. Specifically, we prove that the statement of \Cref{thm:mainResult} remains true for softmax transformers at the expense of a slight increase in the number of blocks.

\begin{theorem}[Exact interpolation with softmax transformers]
\label{thm:mainResultSoftmax}
    Fix $N\in \N$, $d\geq 2$ and a dataset of sequences $\{(X^j, Y^j)\}_{j\in [N]}$ satisfying \Cref{ass:dataset}.
    There exists $\tau^\star > 0$ such that, for all $0 < \tau < \tau^\star$, there exists a softmax transformer $\T^{\tau}$ 
    with feed-forward layers of width $d' \leq 4$, 
    \begin{equation}\label{eq:mainthmSoftmax}
        L = 2\sum_{j=1}^N m^j + 3N \; \text{blocks,} \quad \text{and} \quad 
        P = \mathcal{O}\left(d \sum_{j=1}^N m^j\right)\; \text{parameters}
    \end{equation}
    such that 
    $$
    \T^\tau (X^j) = Y^j \quad \forall j \in [N].
    $$ 
\end{theorem}

Several comments about \Cref{thm:mainResultSoftmax} are in order. 
First, to the best of our knowledge this is the first \emph{exact} interpolation result for deep softmax transformers with real-valued outputs (see \Cref{ss:relatedWork} for further discussion). Moreover, we derive explicit complexity bounds in terms of both the number of blocks and the parameter count. 

Second, the theorem does not follow directly from the hardmax result by taking the ``zero-temperature limit'' of the softmax \cite{elfadelSoftmax1993}. Such an argument would only yield \emph{approximate} interpolation, with errors depending on~$\tau>0$. Our key contribution is to eliminate these errors exactly by exploiting the properties of ReLU feed-forward layers, at the cost of only $N-1$ additional blocks. 

Finally, the temperature can be normalized to $\tau^\star = 1$ by rescaling the self-attention parameter matrix in \eqref{eq:softmaxFormulation}. However, this rescaling comes at the expense of increasing the norm of the parameters in the self-attention layer. Bounding parameter norms is a substantially more delicate problem, which we do not attempt to solve. In practice, such norms depend nontrivially on the dispersion of the data and the separation of the outputs \cite{kim2023provable}. See \Cref{sss:paramNorms} for further discussion.

\subsection{Related work}\label{ss:relatedWork}

The interpolation capacity of transformers has been studied in previous works. In \cite{kim2023provable}, exact interpolation results were established for sequence-to-sequence maps with equal input and output lengths ($m = n$) and integer-valued outputs. The estimates in that work show that the required number of blocks grows as $\Tilde{\mathcal{O}}(n + \sqrt{Nn})$, where the $\Tilde{\mathcal{O}}$ notation hides logarithmic factors in $N$, $n$, and $d$. For real-valued outputs, however, the analysis yields only \emph{approximate} interpolation \cite[Remark~3.2]{kim2023provable}. In \cite{kajitsuka2024are}, the analysis of \cite{kim2023provable} is extended to vector-valued discrete output sequences using one-layer self-attention with low-rank weight matrices. Regarding next-token prediction capacity of shallow decoder transformers, \cite{madden2025next} prove an exact interpolation result where outputs are sequences of probability vectors over a finite vocabulary (see \Cref{sss:extensionProbVectorOutputs} for further comments on this setting). The interpolation properties of transformers have also been analyzed from a controllability-based viewpoint in \cite{agrachev2025generic}, where it is proven that exact interpolation is possible using transformers with more general self-attention layers than those commonly used in practice. In contrast, we prove exact interpolation for real-valued vector outputs in sequence-to-sequence tasks using deep single-head transformers with the standard self-attention mechanism. This requires a number of blocks that grows linearly in both $N$ and $n$, but crucially remains independent of $d$.

A different line of work is pursued in \cite{mahdavi2024memorization}, which studies interpolation in transformers with a single multi-head attention layer. There, the number of heads plays a role analogous to the width of feed-forward layers, and the authors show that $N$ sequences of length $n$ can be memorized with at least $\mathcal{O}(N/n)$ heads.
Our analysis focuses instead on single-head deep transformers, a structurally different architecture that leads to complementary results (see \Cref{sss:multihead} for further discussion on the multi-head case).

Finally, beyond interpolation, there are several works investigating the function approximation capabilities and dynamical behavior of transformers. For instance, the asymptotic dynamics of hardmax attention have been shown to induce clustering phenomena in token representations \cite{alcalde2025clustering}. 
The seminal result of \cite{yunAreTransformersUniversal2020} established that transformers are universal approximators of sequence-to-sequence maps. The work~\cite{alberti2023sumformer} extends these results to efficient variants of transformers, while \cite{furuya2025transformers} proves that transformers are universal approximators of ``in-context maps'' (see that work for a definition of this term). Yet another perspective is offered by \cite{geshkovski2024measure}, where transformers are interpreted as interacting particle systems within a measure-theoretic framework. In this setting, transformers achieve approximate interpolation in the Wasserstein-2 metric for measure-to-measure maps. The construction in \cite{geshkovski2024measure} can also solve \Cref{pb:exactInterp}, but requires $\mathcal{O}(Nd+\sum_{j=1}^N m^j)$ blocks and is limited to the case of datasets where inputs and outputs have the same length $m^j = \len(X^j) = \len (Y^j)$. \Cref{thm:mainResultSoftmax} improves upon these works by achieving lower block complexity and accommodating arbitrary output sequence lengths. 

\subsection{Parameter complexity of transformers versus ResNets}\label{ss:paramComplexity}
Since transformers can be viewed as an extension of ResNets, it is natural to compare our exact interpolation results with analogous strategies for ResNets. We discuss two possible approaches.

As a first approach, we remove self-attention layers entirely, which recovers a feed-forward only transformer 
$$
    \T_{\mathrm{FF}}: (\R^d)^n \to (\R^d)^n, \qquad \T_{\mathrm{FF}}(X) = \{\varphi(x_1), \dots, \varphi(x_n)\},
$$
where $\varphi : \R^d \to \R^d$ is realized by a (possibly deep) ResNet. This architecture acts independently on each token, which poses a fundamental limitation for solving  \Cref{pb:exactInterp} under \Cref{ass:dataset}. 
To see this, consider two input sequences that differ only in their last token,
$$
X^1 = \{ x_1, \ldots, x_{n-1}, x^1_n \}, 
\qquad 
X^2 = \{ x_1, \ldots, x_{n-1}, x^2_n \},
$$
with $x^1_n \neq x^2_n$, and corresponding output sequences
$$
Y^1 = \{ y_1^1, \ldots, y_n^1 \}, 
\qquad 
Y^2 = \{ y_1^2, \ldots, y_n^2 \},
$$
with $y_i^1 \neq y_i^2$ for all $i\in [n]$. Then, 
$$
    \T_{\mathrm{FF}}(X^j) = \{\varphi(x_1), \ldots, \varphi(x_{n-1}), \varphi(x^j_n)\}, \quad j=1,2,
$$
so one can never simultaneously have $\T_{\mathrm{FF}}(X^1) = Y^1$ and $\T_{\mathrm{FF}}(X^2) = Y^2$. Thus, \Cref{pb:exactInterp} cannot be solved by feed-forward only transformers in general.

We next discuss a second approach, which ignores sequence structure altogether. Setting $n = \max_{j\in[N]} n^j$, $m = \max_{j\in[N]} m^j$, and padding shorter sequences with zeros, each input $X^j$ can be regarded as a vector in $\R^{dn}$. Similarly, each output $Y^j$ can be viewed as a vector in $\R^{dm}$. 
In this setting, ResNets operate on $\R^{dn}$, so their parameter count necessarily scales with $n$. For example, the ResNets constructed in \cite{domenec2023NODES} require $\mathcal{O}(Nd n)$ parameters to interpolate the dataset, while our transformers require $\mathcal{O}(d\sum_{j = 1}^N m^j)$ parameters. This comparison yields two regimes:
\begin{itemize}
    \item If $m^j=n$ for all $j\in[N]$, transformers and ResNets have the same parameter order.
    \item If some $m^j<n$, transformers are strictly more efficient than ResNets.
\end{itemize}
The efficiency gain of transformers for shorter output sequences stems from the clustering effect of self-attention (see, e.g., \cite{geshkovski2023emergence, geshkovski2025mathematical, alcalde2025clustering, karagodin2024clustering, burger2025analysis}). Tokens within a sequence are carefully collapsed to be equal to the output sequence in the set sense. ResNets, in contrast, rely only on geometric correlations in the input, so parameter reductions can be established only probabilistically~\cite{alvarez2025clusterClassification}.

The key advantage of transformers is that they combine the token-level processing of ResNets with self-attention driven for separation and clustering. This structure enables exact interpolation for sequences of arbitrary input length, with parameter complexity depending only on the total output length. The tradeoff, which we do not address here, is the quadratic computational cost of self-attention with respect to the input sequence length \cite{vaswaniAttentionAllYou2017}.

\subsection{Organization of the paper}
The remainder of the paper is organized as follows. In \Cref{sec:implicationsTraining} we examine the implications of our exact interpolation results for regularized training. We then introduce the transformer architecture we analyze in \Cref{sec:tfArchitecture}. \Cref{sec:technicalResults} collects auxiliary technical results that will be used to prove \Cref{thm:mainResult} in \Cref{sec:proofOfMainResult} and \Cref{thm:mainResultSoftmax} in \Cref{sec:proofOfMainResultSoftmax}. Finally, \Cref{sec:conclusion} discusses perspectives and open problems.

\section{Implications for transformer regularized training}\label{sec:implicationsTraining}

The existence of an exactly interpolating transformer has important implications for regularized training.  
In this setting, the parameters $\theta \in \R^P$ of a transformer $\T_\theta$ are learned by solving the Tikhonov-regularized problem
\begin{equation}\label{eq:tikhonov.obj}
\min_{\theta \in \mathbb{R}^P} \left\{ F_\varepsilon(\theta) = \frac{1}{N}\sum_{j=1}^N f(\T_\theta (X^j), Y^j) + \varepsilon \kappa(\theta) \right\},
\end{equation}
where $f : \mathcal{X} \times \mathcal{X} \to [0,\infty)$ is a nonnegative lower semicontinuous function such that $f(Y, Y') = 0$ if and only if $Y = Y'$, $\varepsilon > 0$ is a regularization parameter, and the function $\kappa : \mathbb{R}^P \to [0,\infty)$ is lower semicontinuous and coercive (i.e. $\lim_{\|\theta\| \to \infty} \kappa(\theta) = \infty$). These assumptions ensure the minimum in \eqref{eq:tikhonov.obj} is attained by a minimizer $\theta_\varepsilon^\star$.

The unconstrained minimization in \eqref{eq:tikhonov.obj}, which is solved when training transformers in practice, is closely connected to the constrained problem
\begin{equation}\label{eq:tikhonov.constObj}
        \min_{\theta \in \mathbb{R}^P} \kappa(\theta) \quad \text{s.t. }  \T_\theta(X^j) = Y^j \quad \forall j \in [N],
    \end{equation}
as established by the following standard result in Tikhonov regularization theory (see, e.g., \cite{Engl1996Regularization}). We include the proof to make this paper self-contained.

\begin{prop}\label{prop:tikhonovProposition}
Suppose there exists $\theta_0\in \R^P$ and a softmax transformer $\T_{\theta_0}$ satisfying that $\T_{\theta_0}(X^j) = Y^j$ for all $j \in [N]$. Then:
\begin{enumerate}
    \item[i)] There exists $\theta^\star_0$ solving \Cref{eq:tikhonov.constObj}.
    \item[ii)] $F_\varepsilon(\theta^\star_\varepsilon) \leq \varepsilon \kappa(\theta^\star_0)$. In particular, $\kappa(\theta^\star_\varepsilon) \leq \kappa(\theta^\star_0)$.
    \item[iii)] Any sequence of optimizers $\{\theta^\star_\varepsilon\}_{\varepsilon > 0}$ has a converging subsequence.
    \item[iv)] The limit of any converging subsequence $\{\theta^\star_\varepsilon\}_{\varepsilon > 0}$ solves \Cref{eq:tikhonov.constObj}.
\end{enumerate}    
\end{prop}
\begin{proof}
To prove i), note that any minimizing sequence $\{\theta_n\}$ for problem \cref{eq:tikhonov.constObj} remains in the sublevel set $\{\theta: \kappa(\theta)\leq \kappa(\theta_0)\}$, which is compact because $\kappa$ is coercive and lower semicontinuous. Pass to a convergent subsequence (not relabeled) and denote its limit by $\theta_0^*$. This point is feasible for \cref{eq:tikhonov.constObj} because the map $\theta \mapsto \T_\theta(X^j)$ is continuous for every fixed $X^j$. It also satisfies $\kappa(\theta_0^*) \leq \liminf_n(\theta_n)$ because $\kappa$ is lower semicontinuous. Then, since $\{\theta_n\}$ is a minimizing sequence, $\theta_0^*$ must be an optimal solution of \cref{eq:tikhonov.constObj}.

Next, we prove ii). Since $\theta^\star_0$ is admissible but not necessarily optimal for \Cref{eq:tikhonov.obj}, we obtain
\[
F_\varepsilon(\theta^\star_\varepsilon) \leq F_\varepsilon(\theta^\star_0) = \frac{1}{N}\sum_{j=1}^N f(\mathrm{T}_{\theta_0^\star}(X^j), Y^j) + \varepsilon \kappa(\theta^\star_0) = \varepsilon \kappa(\theta^\star_0).
\]
The second equality is true because $\mathrm{T}_{\theta_0^\star}(X^j) = Y^j$ and $f(Y^j, Y^j) = 0$. 

Finally, we prove iii) and iv). Since $f$ is nonnegative, we have that $
\varepsilon \kappa(\theta^\star_\varepsilon) \leq F_\varepsilon(\theta^\star_\varepsilon)$, which by part ii) implies that $\kappa(\theta^\star_\varepsilon) \leq \kappa(\theta^\star_0)$. Thus, $\{\theta^\star_\varepsilon\}_{\varepsilon > 0}$ is uniformly bounded and has a converging subsequence $\{\theta_{\varepsilon_j}\}$. Let $\theta^*$ be its limit. Since $\theta \mapsto \mathrm{T}_\theta(X^j)$ is continuous and $f$ is lower semicontinuous, the function 
\[
\theta \mapsto \sum_{i=1}^N f(\mathrm{T}_\theta(X^j), Y^j)
\]
is lower semicontinuous. We then find that
\[
0 \leq \sum_{i=1}^N f(\mathrm{T}_{\theta^*}(X^j), Y^j) \leq \liminf_{\varepsilon_j \to 0} F_{\varepsilon_j}(\theta^\star_{\varepsilon_j}) \leq \liminf_{\varepsilon_j \to 0} \varepsilon_j \|\theta^\star_0\| = 0.
\]
We deduce that $f(\mathrm{T}_{\theta^*}(X^j), Y^j) = 0$ for all $j \in [N]$, so $\mathrm{T}_{\theta^*}(X^j) = Y^j$ for all $j \in [N]$. Thus, $\theta^*$ is feasible for problem \cref{eq:tikhonov.constObj}. It is also optimal because $\kappa$ is lower semicontinuous and $\kappa(\theta^*_{\varepsilon_j}) \leq \kappa(\theta_0^*)$, so $\kappa(\theta^*) \leq \liminf \kappa(\theta^\star_{\varepsilon_j}) \leq \kappa(\theta^\star_0)$. The proof is complete.
\end{proof}

Several remarks on \Cref{prop:tikhonovProposition} are in order. 
First, the statement has been formulated for softmax transformers since their continuous dependence on $\theta$, when composed with a suitable lower semicontinuous function $f$, guarantees the lower semicontinuity of $\theta \mapsto F_\varepsilon(\theta)$. Consequently, a minimizer of \Cref{eq:tikhonov.obj} exists in this setting. In contrast, the discontinuity inherent to hardmax transformers prevents such a guarantee.

Second, \Cref{thm:mainResultSoftmax} establishes the existence of a parameter $\theta_0$ that exactly interpolates the dataset of $N$ sequences, provided the softmax transformer has width $d' \geq 4$ and $L \geq 2 \sum_{j=1}^N m^j + 3N$ blocks. As a result, the conclusions of \Cref{prop:tikhonovProposition} apply directly to this model class.

Finally, although the content of \Cref{prop:tikhonovProposition} is relatively standard, it carries practical implications for the training dynamics. In particular, if optimization converges to a global minimizer of the regularized objective \Cref{eq:tikhonov.obj}, then the training loss scales at worst linearly with the regularization parameter $\varepsilon$ as $\varepsilon \to 0$. Conversely, deviations from this behavior indicate either that the optimization procedure converges only to a local minimizer, or that exact interpolation is not possible with the chosen model architecture. We illustrate this phenomenon for transformer models with a synthetic example\footnote{Code used to reproduce the
example can be found at: \url{https://github.com/DCN-FAU-AvH/exactInterpTF}}.
\begin{figure}
    \centering
    \includegraphics[width=0.49\linewidth]{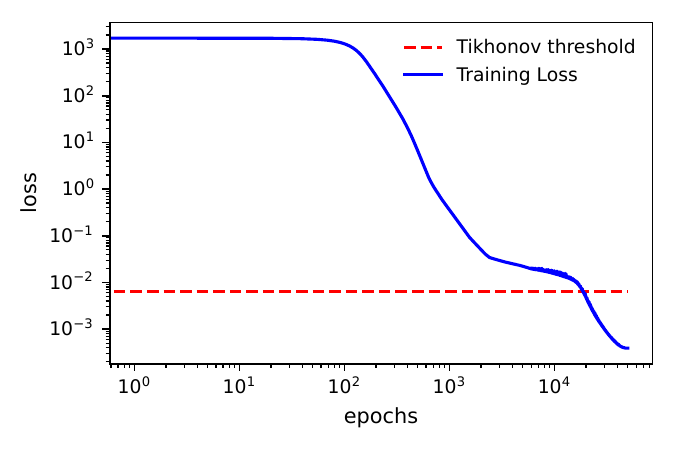}\hfill
    \includegraphics[width=0.49\linewidth]{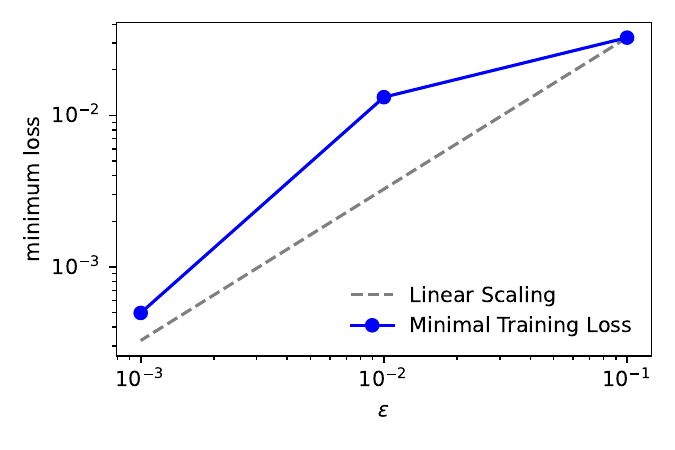}
    \caption{In the left panel, the training loss (log scale) for \Cref{ex:tikhonov}. The red dotted line corresponds to the Tikhonov threshold $\varepsilon \|\theta_{\text{exact}} \|_2^2$. In the right panel, the minimum of the training loss for different choices of the regularization parameter $\varepsilon$.}
    \label{fig:tikhonovLoss}
\end{figure}

\begin{example}\label{ex:tikhonov}
We construct a synthetic dataset using a fixed softmax transformer architecture with a single block, $\tau = 1$ and parameters $\theta_{\text{exact}}$ satisfying $\kappa (\theta_{\text{exact}}) = 62.53$. We generate $N = 10$ input sequences of length $n = 100$ in $\R^{16}$, each paired with an output sequence of length $m = 1$. A transformer with the same architecture is then trained to interpolate this data by minimizing the regularized objective \eqref{eq:tikhonov.obj}.

Since each output sequence consists of a single point, we define the loss function as
$$
f(X, Y = \{y\}) = \left(d_H(X, Y = \{ y\})\right)^2 = \max_{x\in X} \| x - y\|_2^2,
$$
which is easily verified to be continuous (and hence lower semicontinuous). We take $\kappa (\theta) = \| \theta \|_2^2$ as the $\ell^2$ norm in the vectorized parameter space $\R^P$. We train the model with Adam \cite{kingma2014adam} for $5 \cdot 10^4$ optimization steps (epochs), and regularization parameter $\varepsilon = 10^{-4}$. 

To evaluate whether the optimization is approaching a global minimizer, we compare the training loss to the theoretical upper bound $\varepsilon \|\theta_{\text{exact}} \|_2^2$ deduced from \Cref{prop:tikhonovProposition}(ii) for our choice of $\kappa$. As shown in the left panel of \Cref{fig:tikhonovLoss}, the regularized loss falls below this bound only past approximately $10^4$ epochs. We repeat the experiment for various values of the regularization parameter $\varepsilon \in \{10^{-1}, 10^{-2}, 10^{-3}\}$, recording the minimal loss achieved by the optimizer for each. The results, displayed in the right panel of \Cref{fig:tikhonovLoss}, exhibit an approximately linear decay consistent with the theoretical prediction in \Cref{prop:tikhonovProposition}. We conclude that the training process is converging either to a global minimizer, or to a local minimizer with the same $\varepsilon$-scaling for the range of $\varepsilon$ we considered.
\end{example}

\section{The transformer architecture}\label{sec:tfArchitecture}

In this section, we describe the transformer architecture used to solve \Cref{pb:exactInterp}. 
The transformer is defined as a composition of \emph{transformer blocks}, in which a feed-forward (FF) layer is followed by a self-attention (SA) layer. Both layers may or may not include a residual connection. We define these layers in detail next.

\subsection{Feed-forward layers} 
Given a width $d'\in \N$, feed-forward layers are functions $\mathbb{FF} : \mathcal{X} \to \mathcal{X}$ parametrized by $\eta \in \R$, $W \in \R^{d \times d'}$, $U \in \R^{d'\times d}$, $b\in \R^{d'}$, and an activation function $\sigma$. We fix $\sigma (x) = \max (0,x)$ as the usual ReLU activation, and call $d'\in \N$ the \textit{width} of the feed-forward layer. Then, the $i$-th component of the sequence 
\begin{equation*}
\mathbb{FF}(X) = \{ \mathrm{FF}(x_1), \dots, \mathrm{FF}(x_{\len(X)})\} 
\end{equation*}
is given by
\begin{equation*}
\FF (x) = \eta x + W\sigma(U x + b).
\end{equation*}
We stress that allowing for a tunable residual parameter $\eta$, rather than fixing $\eta = 1$, is crucial to obtain the exact interpolation results in \Cref{thm:mainResult} and \Cref{thm:mainResultSoftmax}. In particular, it allows the composition of feed-forward layers to behave as multi-layer perceptrons (MLPs) and ResNets, whose interpolation capabilities have been studied extensively \cite{hernandez2024deep,domenec2023NODES, alvarez2024interplay}.

\subsection{Self-attention layers}\label{sss:SAlayers}

Let $\mathcal{V}, \mathcal{A}$ be two subspaces of $\R^{d\times d}$. Self-attention layers are defined as functions $\mathbb{SA} :\mathcal{X} \to \mathcal{X}$ parametrized by a scalar $\rho \in \R$, a matrix $V\in \mathcal{V}$ and a matrix $A\in \mathcal{A}$ as follows. The $i$-th component of the sequence 
$$
\mathbb{SA}(X) = \{\mathrm{SA}_1(X),\dots, \mathrm{SA}_{\len(X)}(X)\}
$$
is given by
\begin{equation}
\label{eq:selfatt_a}
    \mathrm{SA}_i(X) = \rho x_i + V \sum_{\ell = 1}^{\len(X)} \pi_{i\ell} (X,A) x_\ell,
\end{equation}
where $\pi = (\pi_{i\ell}) \in [0,1]^{\len(X) \times \len(X)}$ is a weighting function parametrized by $A$ and whose output depends on the \textit{entire} sequence $X$.

The usual \textit{softmax self-attention} is obtained with $\pi = \pi^\tau$, where $\tau>0$ is called the ``temperature parameter'' and
\begin{equation}
    \label{eq:softmaxFormulation}
    \pi_{i\ell}^\tau (X,A) = \frac{\exp{\left(\frac{1}{\tau}\langle A x_i, x_\ell \rangle\right)}}{\sum_{k=1}^{\len(X)} \exp{\left(\frac{1}{\tau}\langle A x_i,  x_k \rangle\right)}}.
\end{equation}
Transformer implementations typically parameterize the attention scores using \textit{query} and \textit{key} matrices $Q,K \in \R^{h\times d}$, $h\leq d$ \cite{vaswaniAttentionAllYou2017}, with similarities given by $\langle Qx_i, Kx_\ell\rangle$. Our formulation absorbs these projections into a single matrix $A = K^\top Q$ and reduces precisely to this setting when $\mathcal{A}$ is the subspace of $d\times d$ matrices with rank $h$. 
 
By taking $\tau \to 0$ for fixed $X$ and $A$, one obtains the so-called \textit{hardmax self-attention} formulation~\cite{elfadelSoftmax1993}. We denote it by $\pi = \pi^0$, where 
\begin{equation}
    \label{eq:hardmaxFormulation}
    \pi_{i \ell }^0 (X,A) = 
    \begin{cases}
        \frac{1}{\abs{ \mathcal{C}_i(X,A) }} &\text{if } \ell \in \mathcal{C}_i(X,A),\\
         0 &\text{otherwise}, \\
    \end{cases}
\end{equation}
and 
\begin{equation*}
    \mathcal{C}_i(X,A) = \left\{ j :\;  \langle A x_i, x_j \rangle = \max_{\ell \in [\len(X)]} \langle A x_i, x_\ell \rangle \right\}
\end{equation*} 
is the set of indices $j$ where the inner product $\ip{Ax_i}{x_j}$ is maximized.
\begin{figure}
    \centering
    \includegraphics[width=0.35\linewidth]{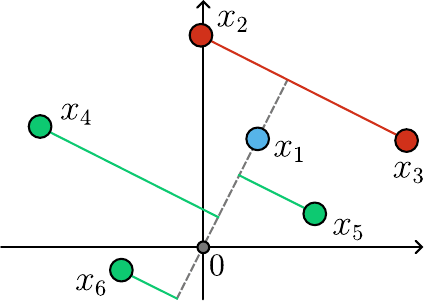}
    \caption{Geometric interpretation of \Cref{eq:hardmaxFormulation} for $i=1$ with $A=I$. Tokens $x_2$ and $x_3$ have the largest orthogonal projection onto $Ax_1 = x_1$, so $\mathcal{C}_i(X,A) = \{2,3\}$.}
    \label{fig:hm_att}
\end{figure}

Although rarely used in practice, we use the hardmax self-attention formulation in our proofs due to its simple geometric interpretation: a token $x_i$ is influenced by the tokens with the largest orthogonal projection onto the direction of $A x_i$ (cf. \Cref{fig:hm_att}). 
The smooth softmax self-attention formulation is usually preferred in applications because it allows for training with standard gradient-based algorithms.

\subsection{The transformer}\label{ss:theTransformer}
We are now ready to combine FF and SA layers into a transformer. Fix a depth $L\in \N$ and a width $d'\in \N$. Given FF layers $\{ \mathbb{FF}^k \}_{k\in [L]}$ with parameters $\{ \eta^k, W^k, U^k, b^k\}_{k\in [L]}$ and SA layers $\{\mathbb{SA}^k \}_{k\in [L]}$ with parameters $\{ \rho^k, V^k, A^k \}_{k\in [L]}$, a transformer is a map $\T : \mathcal{X} \to \mathcal{X}$ defined by
\begin{equation}\label{eq:tf-model}
\T = \left( \mathbb{SA}^L \circ \mathbb{FF}^L \right) \circ \dots \circ \left(\mathbb{SA}^1 \circ \mathbb{FF}^1 \right).
\end{equation}
For each $k\in [L]$, we call the function $\mathbb{TB}^k = \mathbb{SA}^k \circ \mathbb{FF}^k$ a transformer \textit{block} and say that $\T$ in \Cref{eq:tf-model} has $L$ blocks (see \Cref{fig:transformerArchitecture}). 
Note that, from a mathematical perspective, the order of composition of SA and FF layers is irrelevant because one can always choose the parameters of these layers to make them act like the identity map. In particular, all of our results remain true if the transformer blocks are defined as $\mathbb{TB}^k = \mathbb{FF}^k \circ \mathbb{SA}^k$, as in the work introducing the transformer \cite{vaswaniAttentionAllYou2017}.

We will refer to a transformer $\T^0$ with hardmax self-attention layers as a \textit{hardmax transformer}. Similarly, a \textit{softmax transformer}, denoted by $\T^\tau$, will be a transformer with softmax self-attention layers and temperature parameter $\tau > 0$.

In this work, we fix the $d\times d$ matrix subspaces 
\begin{align*}
    \mathcal{V} = \{\alpha I_d: \; \alpha \in \R \}, \qquad  \mathcal{A} = \{\alpha I_d + uv^\top: \; \alpha \in \R,\; u, v \in \R^d\}. 
\end{align*}
This implies that our transformer has $\mathcal{O}(d \sum_j m^j)$ parameters. Of course, our construction also works if one replaces these matrix subspaces with larger ones.

\section{Useful technical results}\label{sec:technicalResults}

In this section, we collect a number of auxiliary technical results that will be instrumental in the proof of our main results. We first revisit the analysis in \cite{alcalde2025clustering} concerning the asymptotic behavior of tokens evolving under repeated application of hardmax self-attention layers described in \eqref{eq:selfatt_a} with $\pi = \pi^0$. 
This asymptotic analysis provides a simplified setting in which we can gain insight into key features of self-attention layers. Beyond being of independent interest as an extension of previous work, the results of this analysis will also inform the design of self-attention layers in the proofs of \Cref{thm:mainResult} and \Cref{thm:mainResultSoftmax}.
We then move on to structural lemmas that describe how self-attention layers interact with feed-forward layers inside a transformer block, highlighting geometric constructions that will be used repeatedly in the sequel.

\subsection{Asymptotics of hardmax self-attention layers revisited}\label{ss:asymptoticsRevisited}

Our goal is to sharpen the characterization of the equilibria of the discrete-time dynamics
\begin{equation}\label{eq:hardmaxDynamics}
    x_i(k+1) = \rho x_i(k) + \frac{1}{|\mathcal{C}_i(X(k), A)|} V
        \sum_{\ell \in \mathcal{C}_i(X(k), A)} x_\ell(k),
\end{equation}
where $\mathcal{C}_i(X(k), A)$ denotes the hardmax-selected cluster for token $i$ at time $k$. In \cite{alcalde2025clustering}, convergence was established in the special case where $\rho = 1 - \gamma$, $V = \gamma I_d$, $\gamma \in (0,1)$, with $A \in \R^{d\times d}$ symmetric positive definite. However, the set of possible equilibria in the full-rank positive definite case is extremely rich: it ranges from a fully concentrated state to a fully dispersed one \cite{geshkovski2023emergence, alcalde2025clustering}. This motivates us to study additional structural assumptions on $A$ and on the initial configuration to enforce a desired concentration.
\begin{figure}
    \centering
    \includegraphics[width=0.9\linewidth]{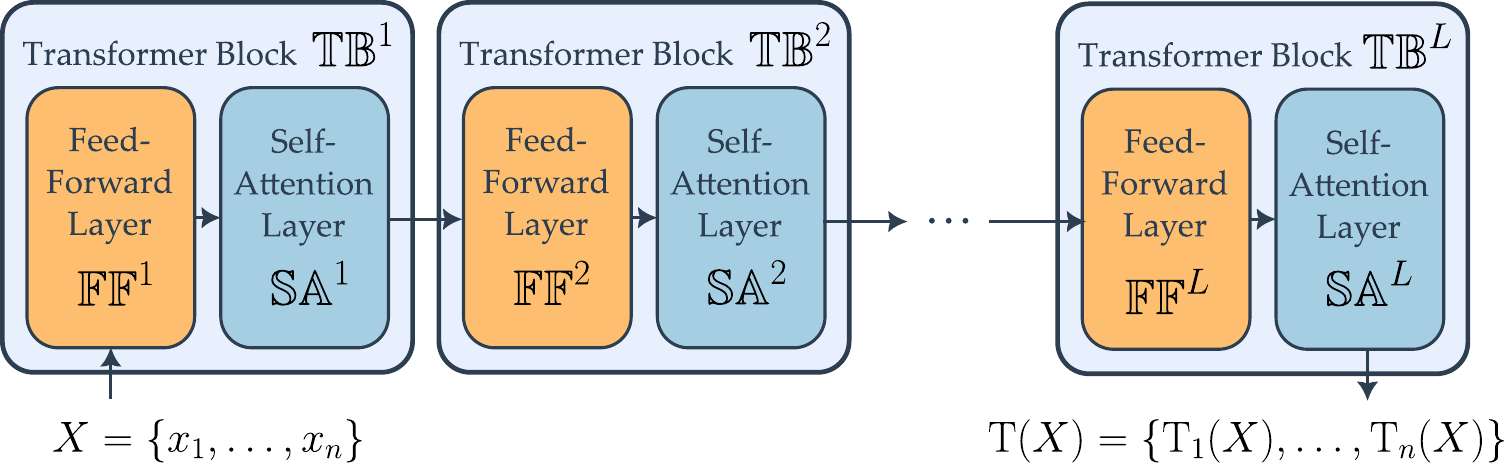}
    \caption{Schematic of the transformer architecture described in \Cref{ss:theTransformer}.}
    \label{fig:transformerArchitecture}
\end{figure}

\subsubsection*{Rank-one parameter matrix $A$}

A natural way to restrict the equilibria is to reduce the rank of $A$. We only consider the symmetric rank-one case $A = v v^\top$,  where $v \in \R^d \setminus \{0\}$.

\begin{lemma}\label{lem:rank1-asymptotics}
Given an initial configuration $\{ x_i(0) \}_{i \in [n]}$, suppose $v\in \R^d$ satisfies $\ip{v}{x_i(0)} \neq 0$ for all $i \in [n]$, and assume there exist unique indices $m,M\in [n]$ such that $\ip{v}{x_M(0)} > 0$, $\ip{v}{x_m(0)} < 0$ and 
$$
    \ip{v}{x_M(0)} > \ip{v}{x_j(0)} \quad \forall j \neq M,
    \qquad 
    \ip{v}{x_m(0)} < \ip{v}{x_j(0)} \quad \forall j \neq m.
$$
Then, the dynamics governed by \eqref{eq:hardmaxDynamics} with $\rho = 1 - \gamma$, $V= \gamma I_d$, $\gamma\in (0,1]$, and $A = vv^\top$ converge to an equilibrium consisting of $n_1$ copies of $x_m(0)$ and $n_2$ copies of $x_M(0)$ where $n_1 + n_2 = n$. For $\gamma=1$, convergence happens in one step.
\end{lemma}
\begin{proof}
If $\gamma=1$, the result follows from a direct calculation using  \eqref{eq:hardmaxDynamics} with $\rho=0$ and $V= I_d$.
If $\gamma < 1$, we start by proving that the following statements hold for all $k\geq 0$:
\begin{align}
&\argmax_{j\in[n]} \langle v, x_j(k)\rangle = \{M\}, \label{eq:(a)} \\
&\argmin_{j\in[n]} \langle v, x_j(k)\rangle = \{m\}, \label{eq:(b)} \\
&x_M(k) = x_M(0) \quad \text{and} \quad x_m(k) = x_m(0), \label{eq:(c)} \\[1em]
&\operatorname{sign}(\langle v, x_i(k)\rangle) = 
\operatorname{sign}(\langle v, x_i(0)\rangle) 
\quad \text{for all } i\in[n]. \label{eq:(d)}
\end{align}
These properties hold at $k=0$ by the assumptions. We now assume they hold at time $k$ and prove them at time $k+1$. First, we observe that
$\langle A x_i(k), x_j(k)\rangle
= \langle v, x_i(k)\rangle \langle v, x_j(k)\rangle$,
so
\[
\mathcal{C}_i(X(k),A)
=
\begin{cases}
\{M\}, & \langle v, x_i(k)\rangle > 0,\\
\{m\}, & \langle v, x_i(k)\rangle < 0.
\end{cases}
\]
This is because maximizing $\langle v, x_i(k)\rangle\langle v, x_j(k)\rangle$ over $j$ amounts to maximizing
$\langle v, x_j(k)\rangle$ if $\langle v, x_i(k)\rangle>0$ and minimizing it if $\langle v, x_i(k)\rangle<0$, and by the induction hypothesis the unique maximizer and minimizer are $M$ and $m$.

Next, we compute $x_i(k+1)$ using the update rule \eqref{eq:hardmaxDynamics} with $\rho=1-\gamma$ and $V=\gamma I_d$, and take the inner product with $v$ to find that
\begin{equation}\label{e:next-ips}
\langle v, x_i(k+1)\rangle =
\begin{cases}
(1-\gamma)\langle v, x_i(k)\rangle + \gamma \langle v, x_M(k)\rangle
& \text{if }\langle v, x_i(k)\rangle>0,\\
(1-\gamma)\langle v, x_i(k)\rangle + \gamma \langle v, x_m(k)\rangle
& \text{if }\langle v, x_i(k)\rangle<0.
\end{cases}
\end{equation}
Property \Cref{eq:(d)} at time $k$ and our assumptions that $\langle v,x_m(0)\rangle < 0 < \langle v,x_M(0)\rangle$ imply that $\langle v,x_m(k)\rangle < 0 < \langle v,x_M(k)\rangle$, so we also have $\mathcal C_M(X(k),A)=\{M\}$ and $\mathcal C_m(X(k),A)=\{m\}$. Then, since $\gamma \in (0,1]$, we conclude from \cref{e:next-ips} that properties \Cref{eq:(c)} and \Cref{eq:(d)} hold at time $k+1$. We also conclude that
\begin{align*}
\langle v,x_i(k+1)\rangle-\langle v,x_M(k)\rangle
&=(1-\gamma)\bigl(\langle v,x_i(k)\rangle-\langle v,x_M(k)\rangle\bigr) <0&\text{if }\langle v,x_i(0)\rangle>0,
\\
\langle v,x_i(k+1)\rangle-\langle v,x_m(k)\rangle
&=(1-\gamma)\bigl(\langle v,x_i(k)\rangle-\langle v,x_m(k)\rangle\bigr)>0
&\text{if }\langle v,x_i(0)\rangle<0.
\end{align*}
These inequalities imply that $\langle v,x_m(k+1)\rangle < \langle v,x_i(k+1)\rangle <\langle v,x_M(k+1)\rangle$, which in turn imply that properties \Cref{eq:(a)} and \Cref{eq:(b)} hold at time $k+1$. This concludes the induction process.

Finally, let $n_1 = \bigl|\{i:\langle v,x_i(0)\rangle<0\}\bigr|$, $n_2 = \bigl|\{i:\langle v,x_i(0)\rangle>0\}\bigr|$. For every $i$ with $\langle v,x_i(0)\rangle>0$ we have
$x_i(k+1)-x_M(0)=(1-\gamma)\bigl(x_i(k)-x_M(0)\bigr)$,
and therefore
\[
x_i(k)=x_M(0)+(1-\gamma)^k\bigl(x_i(0)-x_M(0)\bigr)\to x_M(0).
\]
Similar arguments show that $x_i(k)\to x_m(0)$ for every $i$ with $\langle v,x_i(0)\rangle<0$. 
Consequently the limiting configuration consists of $n_1$ copies of $x_m(0)$ and $n_2$ copies of $x_M(0)$ with $n_1+n_2=n$. This is easily seen to be an equilibrium for the dynamical map \cref{eq:hardmaxDynamics}.
\end{proof}
This result is advantageous if one seeks equilibria with at most two clusters, and its hypotheses hold for almost every $v\in \R^d$ for a given initial configuration $\{x_i(0)\}_{i\in [n]}$. However, when more than two distinct clusters are desired, rank-one interactions are too restrictive, and one must instead consider a different structure for $A$ together with structured initial conditions. 

\subsubsection*{Scalar multiple of the identity $A$}

We now study the case \(A=\xi I_d\) with \(\xi>0\) and dynamics \eqref{eq:hardmaxDynamics} with \(\rho=1-\gamma\), \(V=\gamma I_d\), \(\gamma\in(0,1]\). 
We show that the long-term behavior of the dynamics \eqref{eq:hardmaxDynamics} is highly sensitive to the initial configuration of tokens. 
In particular, by arranging the tokens in specific regions of $\R^d$, one can enforce convergence to qualitatively different equilibria: a fully clustered state (all tokens coincide), a fully dispersed state (all tokens remain distinct), or an intermediate configuration with a prescribed number of clusters. 
The following lemmas illustrate these three possibilities.

In what follows, we denote the positive orthant as $\R_{>0}^d$, the negative orthant as $\R_{<0}^d$, the open hypercube of side length $R>0$ as $H_R = (0,R)^d$ and the vector of all ones in $\R^d$ as $1_d$.

\begin{lemma}[Full clustering]\label{lem:fullConc}
Let $A = \xi I_d$ with $\xi > 0$, \(\rho=1-\gamma\), \(V=\gamma I_d\), $\gamma\in(0,1]$. Let the initial configuration $\{x_i(0)\}_{i\in[n]}$ be such that, for a unique index $i^\star\in [n]$
\begin{enumerate}
    \item[i)] $x_i(0) \in H_R$ for all $i \neq i^\star$
    \item[ii)] $x_{i^\star}(0) = R \; 1_d$.
\end{enumerate}
Then, the dynamics governed by \eqref{eq:hardmaxDynamics} converge to the fully clustered equilibrium
$$
    X^\star = (x_{i^\star}(0), \dots, x_{i^\star}(0)).
$$
For $\gamma = 1$, convergence takes place in one step.
\end{lemma}
\begin{proof}
We first note that 
$\mathcal{C}_{i}(X(0), A) = \{i^\star\}$ for all $i\in [n]$, so we have
$$
    x_{i}(1) = x_i(0) + \gamma (x_{i^\star}(0) - x_i(0)).
$$
For $\gamma = 1$, this proves that $x_i(1) = x_{i^\star}(0)$, and since all tokens are fixed, $x_i(k) = x_{i^\star}(0)$ for all $i\in [n]$ and all $k\geq 1$. 
For the case $\gamma \in (0,1)$, notice that $x_i(1)\in H_R$ for all $i\neq i^\star$ because $H_R$ is convex. By induction on $k \geq 0$, one finds that
$$
    x_{i}(k+1) = x_i(k) + \gamma (x_{i^\star}(0) - x_i(k)),
$$
or equivalently by setting $\delta_i(k) \coloneqq x_i (k) - x_{i^\star} (0)$, $\delta_{i}(k+1) = (1- \gamma) \delta_i(k)$. Then, 
$$
    \| \delta_i (k+1) \| \leq (1- \gamma)^{k+1} \| \delta_i(0)\|,
$$
which implies convergence as $k\to \infty$.
\end{proof}
The next result gives a configuration for which no clustering occurs. We will use $S_R = \{x \in \R^d : \|x\| = R\}$ to denote the hypersphere of radius $R$.

\begin{lemma}[No clustering]\label{lem:noConc}
Let $A = \xi I_d$ with $\xi > 0$, \(\rho=1-\gamma\), \(V=\gamma I_d\), \(\gamma\in(0,1]\). 
Any initial configuration $\{ x_i(0)\}_{i\in [n]}$ of pairwise distinct tokens with $x_i(0) \in S_R$ for all $i\in [n]$ is an equilibrium for \eqref{eq:hardmaxDynamics}.
\end{lemma}
\begin{proof}
Since $x_i(0)\in S_R$ for all $i\in [n]$, then is holds that $\mathcal{C}_i(X(0), A) = \{ i \}$. Then, $x_i(1) = x_i(0)$ for all $i\in [n]$ and the result follows from straightforward induction. 
\end{proof}
\begin{figure}
    \centering
    \includegraphics[width=0.85\linewidth]{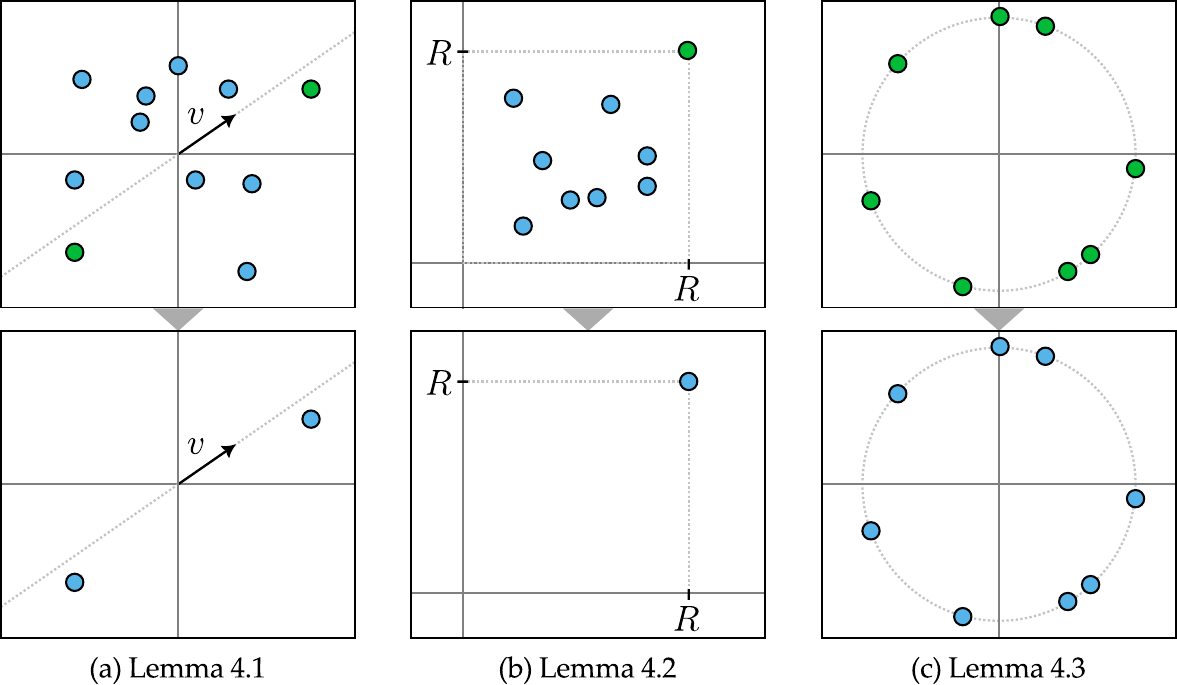}
    \caption{Initial configuration of tokens in the top row, and the asymptotic configuration in the bottom row for $(a)$ \Cref{lem:rank1-asymptotics}, $(b)$ \Cref{lem:fullConc} and $(c)$ \Cref{lem:noConc}.}
    \label{fig:asymptoticLemmasSA}
\end{figure}
An illustration of the initial configurations and their asymptotic configurations for \Cref{lem:rank1-asymptotics,lem:fullConc,lem:noConc} is shown in \Cref{fig:asymptoticLemmasSA}. Finally, we provide a result to obtain partial clustering with a desired degree of concentration.

\begin{lemma}[Partial clustering]\label{lem:partialConc}
Let $A = \xi I_d$ with $\xi > 0$, \(\rho=1-\gamma\), \(V=\gamma I_d\), \(\gamma\in(0,1]\). Let the initial configuration $\{ x_i(0)\}_{i\in [n]}$ contain pairwise distinct tokens and the subset $\mathcal{I} = \{i_1,\dots,i_m\} \subseteq [n]$ with $|\mathcal{I}|=m$ satisfy
\begin{enumerate}
    \item[i)] $x_i(0) \in H_R$ for all $i \notin \mathcal{I}$,
    \item[ii)] $x_{i_1}(0) = R \; 1_d$,
    \item[iii)] $x_{i_j}(0) \in S_R \cap \R_{<0}^d$ for all $j > 1$.
\end{enumerate}
Then the dynamics governed by \eqref{eq:hardmaxDynamics} converge to a permutation of the equilibrium
$$
    X^\star = \big(\underbrace{x_{i_1}(0), \dots, x_{i_1}(0)}_{n-m+1\text{ times}}, 
    \, x_{i_2}(0), \dots, x_{i_m}(0)\big),
$$
that is, a partially clustered state with $m$ distinct tokens.
For $\gamma = 1$, convergence takes place in one step.
\end{lemma}
\begin{proof}
Arguing as in the proof of \Cref{lem:fullConc}, together with the fact that $x_{i_j}(0) \in \R^d_{<0}$ for all $j > 1$, we obtain $\mathcal{C}_i(X(0), A) = \{ i_1 \}$ for all $i\notin \mathcal{I}$, while arguing as in the proof of \Cref{lem:noConc}, we have that $\mathcal{C}_i(X(0), A) = \{ i \}$ for all $i\in \mathcal{I}$. Thus, 
\begin{align*}
    x_i(1) &= x_i(0) + \gamma (x_{i_1}(0) - x_i(0)) \quad \text{for all} \;\; i\notin \mathcal{I}, \\
    x_{i'}(1) &= x_{i'}(0) \quad \text{for all} \;\; i'\in \mathcal{I}.
\end{align*}
This immediately proves the result when $\gamma = 1$. For $\gamma \in (0,1)$, we run an induction argument on $k\geq 1$ as in \Cref{lem:fullConc} to obtain the result. 
\end{proof}

\subsection{Structural lemmas for transformer blocks}\label{ss:preliminaryTfLemmas}

We now turn to transformer blocks and investigate how self-attention and feed-forward layers interact, focusing on structural properties that will be essential for proving \Cref{thm:mainResult}. Our first question is whether a self-attention layer can leave a specific token $x_i$ from a sequence $X^j$ unchanged. (Throughout this section, superscripts always refer to sequence indices in the dataset, while subscripts refer to tokens indices within the sequence. In particular, we write $x^j_i$ to indicate the $i$-th token from sequence $X^j$.)
If $V = (1 - \rho) I_d$, it is not hard to see from \eqref{eq:selfatt_a} that $x_i$ will not be changed by a self-attention layer if $\mathcal{C}_i(X,A) = \{ i \}$. In such a case, we say that the token $x_i\in X$ is a \textit{leader}. 
The following lemma asserts that a token can be a leader if and only if it is extreme for the closed convex hull of the tokens in the sequence $X$. As usual, we denote by $\co (X)$ the closed convex hull of a sequence $X$ and by $\ext \co (X)$ the set of its extreme points.

\begin{lemma}\label{lem:howToLeader}
Let $X$ be a sequence of non-zero tokens. There exists $A\in \R^{d\times d}$ such that $x_i \in X$ is a leader if and only if $x_i \in \ext \co (X)$.
\end{lemma}
\begin{proof}
Suppose $A\in \R^{d\times d}$ is such that $x_i$ is a leader. Then,
\begin{equation}
    \langle A x_i, x_\ell \rangle < \langle A x_i, x_i\rangle \quad \forall \ell \neq i.
\end{equation}
It is easy to verify that this implies that $x_i$ is the unique maximizer of the linear function $x \mapsto \langle Ax_i, x\rangle$ over $\co(X)$. Since linear functions on convex sets are always maximized by at least one extreme point of the set, we conclude that $x_i \in \ext \co (X)$. This proves the ``only if'' part of the lemma.

For the reverse implication, we fix $x_i\in\ext\co (X)$ and explicitly construct $A$ for which $x_i$ is a leader. Since $x_i\in\ext\co (X)$, then $x_i\notin \co (X \setminus \{ x_i \})$, which is a closed and convex set. By the hyperplane separation theorem, there exists a non-zero vector $v\in \R^d$ and a constant $\alpha\in \R$ such that $\ip{v}{x_i} > \alpha$ and $\ip{v}{x_\ell} < \alpha$ for all $\ell\neq i$. Since the maps $v \mapsto \ip{v}{x_\ell}$ are continuous for all $\ell \in [\len (X)]$, we can assume without loss of generality that $\ip{v}{x_i} \neq 0$. 
Then, if $\ip{v}{x_i} > 0$ we set $A = vv^\top$ and verify that
\begin{equation}
    \begin{aligned}
    \langle A x_i, x_i\rangle &= \ip{v}{x_i} \ip{v}{x_i} \\
    &> \ip{v}{x_i}\ip{v}{x_\ell} \\
    &= \langle Ax_i, x_\ell\rangle \quad \forall \ell \neq i,
    \end{aligned}
\end{equation}
because $\ip{v}{x_i} > \alpha > \ip{v}{x_\ell}$. This yields $\mathcal{C}_i(X,A) = \{ i \}$, as desired. If, instead, $\ip{v}{x_i} < 0$, then we set $A = -vv^\top$ and obtain $\mathcal{C}_i(X,A) = \{ i \}$ because
\begin{align}
    \langle A x_i, x_i\rangle &= -\ip{v}{x_i} \ip{v}{x_i} \nonumber \\ 
    &> -\ip{v}{x_i}\ip{v}{x_\ell} \\
    &= \langle Ax_i, x_\ell\rangle \quad \forall \ell \neq i,
\end{align}
which finishes the proof.
\end{proof}
The next two lemmas establish that all tokens in a sequence $X^j$ can be attracted by a single token in that sequence. This is a crucial result for the proof of \Cref{thm:mainResult} and is inspired by the asymptotic dynamics for $A = v v^\top$ in \Cref{lem:rank1-asymptotics}. Moreover, the lemmas allow prescribing the attracting token $x_{i^\star}^{j^\star}$ in a sequence $X^{j^\star}$, which is also used in the proof of \Cref{thm:mainResult} iteratively to separate all sequences one by one.

\begin{lemma}\label{lem:auxLemmaSingleLeader}
Given sequences $X^1, \dots, X^N$ of non-zero tokens, fix $j^\star\in [N]$ and $i^\star\in [\len (X^{j^\star})]$ such that $x_{i^\star}^{j^\star} \in \ext\co(X^{j^\star})$. There exists $v\in \R^d$ such that:
\begin{enumerate}
    \item[i)] $\ip{v}{x_{i^\star}^{j^\star}} \neq 0$.
    \item[ii)] 
    For every $j \in [N]$, there exists $i_j\in [\len (X^j)]$ such that $\ip{v}{x_\ell^j} < \ip{v}{x_{i_j}^j}$ for all $\ell \in [\len (X^j)] \setminus \{i_j\}$.
    \item[iii)] $i_{j^\star} = i^\star$.
\end{enumerate}
\end{lemma}
\begin{proof}
We assume $j^\star = 1$ without loss of generality, and proceed by induction on $N$. For $N=1$, a vector $v_1\in \R^d$ such that conditions \textit{i)--iii)} hold can be chosen using the same arguments in \Cref{lem:howToLeader}. 

Next, we assume that $v_{N-1}\in \R^d$ is such that \textit{i)--iii)} hold for the first $N-1$ sequences, and construct $v_{N}\in \R^d$ such that they also hold when the last sequence $X^N$ is considered. Because $j^\star = 1$, only condition \textit{ii)} can fail for $v = v_{N-1}$ when $X^N$ is added to the problem. If condition \textit{ii)} holds also for $j=N$, we can simply set $v_{N} = v_{N-1}$.

Otherwise, suppose condition \textit{ii)} does not hold for $j=N$. In this case, the quantity $\ip{v}{x_\ell^{N}}$ is maximized at $s>1$ tokens, which we may take to be $x_1^{N}, \dots, x_s^{N}$ without loss of generality after reordering the sequence if necessary. We may also assume without loss of generality that $x_1^{N}$ is extreme for $\co(x_1^{N}, \dots, x_s^{N})$. Then, by the hyperplane separation theorem, there exists $u\in\R^d$ such that $\ip{u}{x_r^{N}} < \ip{u}{x_1^{N}}$ for all $r \in \{2,\dots,s\}$.
Set $v_{N} = v_{N-1} + \varepsilon u$ for some $\varepsilon>0$ to be specified below. Then, we obtain
\begin{equation*}
\begin{aligned}
    \ip{v_{N}}{x_r^{N}} &= \ip{v_{N-1}}{x_r^{N}} + \varepsilon \ip{u}{x_r^{N}}  \\
    &< \ip{v_{N-1}}{x_1^{N}} + \varepsilon \ip{u}{x_1^{N}}  \\
    &= \ip{v_{N}}{x_1^{N}}
\end{aligned}
\end{equation*}
for every $r \in \{2,\dots, s\}$.
We now fix $\varepsilon$ small enough that $\ip{v_{N}}{x_r^{N}} < \ip{v_{N}}{x_1^{N}}$ also for $r\in\{s+1,\ldots,\len(X^N)\}$ and that conditions \textit{i)--iii)} remain true for sequence indices $j\in[N-1]$. This is possible because the maps $v\mapsto \ip{v}{x}$ are continuous.
\end{proof}
Next, we look at how feed-forward layers and the hardmax operation combine. Our goal is to show that alternating feed-forward and self-attention layers can cause all tokens of a sequence to align with a unique leader.

\begin{lemma}\label{lem:chooseLeader}
Fix $j^\star \in [N]$ and $i^\star \in [\len (X^{j^\star})]$ such that $x_{i^\star}^{j^\star} \in \ext\co(X^{j^\star})$. There exist parameters $\{ \eta, W, U, b\}$ of a feed-forward layer $\mathbb{FF}$ with width $d' = 1$, a vector $v\in \R^d$, and indices $\{i_j\}_{j\in [N]}$ with $i_{j^\star} = i^\star$ such that, for all $\{\ell, j\}\in [\len (X^j)]\times [N]$,
\begin{equation}\label{eq:chooseLeaderLemma}
   \FF(x_\ell^j) \neq 0 
   \quad\text{and}\quad
    \C_\ell(\mathbb{FF}(X^j), vv^\top) = \{ i_j \}.
\end{equation}
\end{lemma}
\begin{proof}
Choose $v\in\R^d$ as in \Cref{lem:auxLemmaSingleLeader}. We now choose the parameters of the FF layer such that \Cref{eq:chooseLeaderLemma} holds. For this, we first observe that adding the same constant vector to every token of every sequence does not change properties \textit{ii)--iii)} of \Cref{lem:auxLemmaSingleLeader}. We then set the parameters of $\mathbb{FF}$ as $\eta =1$, $W = v / \| v \|^2$, $U = 0$ and 
\begin{equation*}
    b = \max_{k\in [N]} \max_{r\in [\len(X^k)]}| \ip{v}{x_r^k}| + \varepsilon,
\end{equation*} 
where $\varepsilon > 0$. For this choice, $\sigma(Ux + b) = \sigma (b) = b$, so that $\FF(x_\ell^j) = x_\ell^j + b/ {\| v \|^2} v$. It is clear that, for all $\varepsilon$ sufficiently small, $\FF(x_\ell^j) \neq 0$ for all $\{\ell, j\}\in [\len(X^j)]\times [N]$, yielding the first condition in \Cref{eq:chooseLeaderLemma}. For the second one, note that
\begin{align}\label{eq:lemPosTerms}
    \ip{v}{\FF({x}_\ell^j)} &= \ip{v}{x_\ell^j} +  \frac{\ip{v}{v}}{\|v\|^2} b \geq \varepsilon > 0 
\end{align}
for all $\{ \ell, j\} \in [\len(X^j)] \times [N]$. With this inequality we find that $i\in \C_\ell (\mathbb{FF}(X^j), vv^\top)$ if and only if
\begin{equation*}
\begin{aligned}
    i &\overset{\phantom{\text{by \Cref{eq:lemPosTerms}}}}{\in}\argmax_{r \in [\len (X^j)]} \left[ \ip{v}{\FF(x^j_\ell)} \ip{v}{\FF(x_r^j)} \right] \\
    &\overset{\text{by \Cref{eq:lemPosTerms}}}{=} \argmax_{r \in [\len (X^j)]}  \ip{v}{\FF(x_r^j)}.
\end{aligned}
\end{equation*}
This fact, combined with property \textit{ii)} of \Cref{lem:auxLemmaSingleLeader} implies that $\C_\ell(\mathbb{FF}(X^j), vv^\top) = \{ i_j \}$ for all $\ell \in [\len(X^j)]$ and all $j\in [N]$. Finally, $i_{j^\star} = i^\star$ by property \textit{iii)} of \Cref{lem:auxLemmaSingleLeader}. 
\end{proof}
Finally, the following lemma is related to the behavior of feed-forward layers and their ability to approximate ``hat functions'', moving exactly one token to an arbitrary position, while keeping the rest fixed. This is a distinctive trait of ReLU feed-forward layers and key to arranging tokens in the initial configurations of \Cref{lem:fullConc,lem:partialConc,lem:noConc}. 

\begin{lemma}\label{lem:hatFFClassifier}
    Fix pairwise distinct $x_1, \dots, x_m \in \R^d$, $i\in [m]$, and $y\in \R^d$. There exist parameters of a $\FF$ layer of width $d' = 3$ such that $\FF(x_i) = y$ and $\FF(x_j) = x_j$ for all $j\neq i$.
\end{lemma}
\begin{proof}
Fix the width of the FF layer as $d' = 3$ and parametrize it by
\begin{equation*}
\mathrm{FF}(x) = x + \begin{bmatrix}
\vert & \vert & \vert \\
{w} & -2{w} & {w} \\
\vert & \vert & \vert
\end{bmatrix} \sigma\left( \begin{bmatrix}
\rule{0.4cm}{0.4pt} & {u^\top} & \rule{0.4cm}{0.4pt} \\
\rule{0.4cm}{0.4pt} & {u^\top} & \rule{0.4cm}{0.4pt} \\
\rule{0.4cm}{0.4pt} & {u^\top} & \rule{0.4cm}{0.4pt}
\end{bmatrix} x + \begin{bmatrix}
    {\beta} - {\gamma} \\
    {\beta} \\
    {\beta} + {\gamma}
\end{bmatrix}\right).
\end{equation*}
Let $u\in \R^d$ be an arbitrary non-zero direction satisfying
\begin{equation*}
     \ip{u}{x_j - x_i} \neq 0 \quad \forall j\neq i.
\end{equation*}
Such a vector exists because we have a finite collection of points. Now, fix $\beta = - \ip{u}{x_i}$, and define the hyperplane
\begin{equation*}
    \mathcal{H} = \{ x \in \R^d \; : \ip{u}{x - x_i} = 0 \}.
\end{equation*}
Moreover, define the ``tube''
\begin{equation*}
    \mathcal{H}^{\gamma} = \{ x \in \R^d \; : |\ip{u}{x - x_i}| \leq \gamma \}.
\end{equation*}
By construction, $x_i \in \mathcal{H}$ and every $x_j$, $j\neq i$ satisfies $x_j\notin \mathcal{H}$. Then, for $\gamma > 0$ small enough, $x_j \notin \mathcal{H}^\gamma$. This already yields that $\FF (x_j) = x_j$ for all $x_j$ with $j\neq i$. Finally, we choose $w = (y - x_i) / \gamma$ to obtain $\FF (x_i) = x_i + w \gamma = y$, as desired.
\end{proof}
Finally, we concatenate the two FF layers constructed in \Cref{lem:hatFFClassifier,lem:chooseLeader} into a single FF layer that implements their composition.

\begin{lemma}\label{lem:ff-combo-w4}
    Let $\FF^{1}$ and $\FF^{2}$ be the feed-forward layers from \Cref{lem:hatFFClassifier,lem:chooseLeader}, respectively. There exists a feed-forward layer $\FF$ of width $d'=4$ such that $\FF = \FF^{1}\circ \FF^{2}$.
\end{lemma}
\begin{proof}
    Let $b>0$ and $v$ be chosen as in \Cref{lem:chooseLeader}, and set $c=b \|v\|^{-2} v$. The FF layer from \Cref{lem:chooseLeader} satisfies $\FF^{1}(x)=x+c$. Then, the width-$4$ FF layer
    \begin{equation*}
    \mathrm{FF}(x) = x + \begin{bmatrix}
    \vert & \vert & \vert & \vert \\
    {w} & -2{w} & {w} & c \\
    \vert & \vert & \vert & \vert 
    \end{bmatrix} \sigma\left( \begin{bmatrix}
    \rule{0.4cm}{0.4pt} & {u^\top} & \rule{0.4cm}{0.4pt} \\
    \rule{0.4cm}{0.4pt} & {u^\top} & \rule{0.4cm}{0.4pt} \\
    \rule{0.4cm}{0.4pt} & {u^\top} & \rule{0.4cm}{0.4pt} \\
    \rule{0.4cm}{0.4pt} & {0^\top} & \rule{0.4cm}{0.4pt} \\
    \end{bmatrix} x + \begin{bmatrix}
        {\beta} - {\gamma} \\
        {\beta} \\
        {\beta} + {\gamma} \\
        b
    \end{bmatrix}\right)
    \end{equation*}
    satisfies $\FF(x)
    =\FF^{2}(x)+c
    =\FF^{1}(\FF^{2}(x))$, as claimed.
\end{proof}

\section{Proof of \texorpdfstring{\Cref{thm:mainResult}}{Theorem \ref{thm:mainResult}}}\label{sec:proofOfMainResult}

We now prove \Cref{thm:mainResult} by prescribing explicit transformer parameters. 
Although the construction is technical, it follows an intuitive strategy that also sheds light on how transformers operate. To emphasize this intuition, we first present the proof strategy in the context of a simple dataset, postponing the full construction to \Cref{ss:separationStep,ss:leaderSelectionStep,ss:collapseStep,ss:interpolationStep,ss:complexityEstimate}.

\subsection{Proof strategy}
Our proof proceeds in four steps, illustrated in \Cref{fig:thm_sketch} for a simple sequence interpolation problem in $\R^2$ with $N=3$ input sequences $X^1$, $X^2$ and $X^3$ of length $n^1 = 5$, $n^2 = 3$ and $n^3 = 3$ and output sequences of length $m^1 = 1, m^2 = 1$, and $m^3 = 2$. The initial configuration of the sequences and the desired output sequences are shown in panel $(a)$ of \Cref{fig:thm_sketch}.
\begin{figure}
    \centering
    \includegraphics[width=1\linewidth]{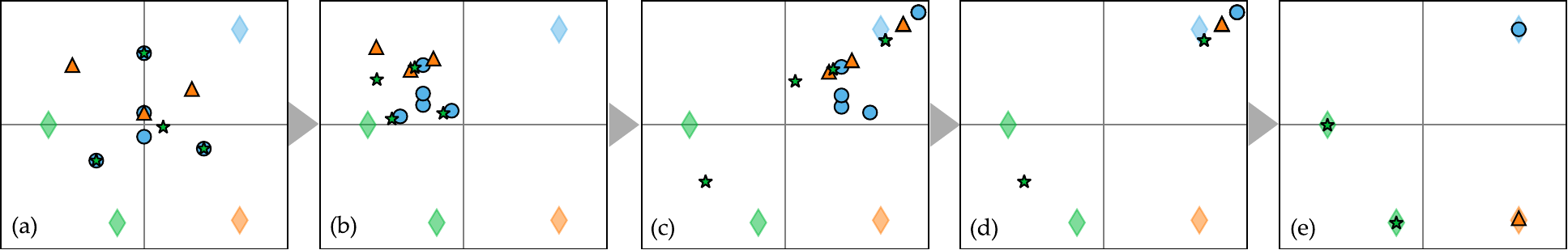}
    \vspace{-3mm}
    \caption{Proof sketch of \Cref{thm:mainResult} applied to $N=3$ input sequences in $\R^2$ (denoted with blue circles, orange triangles and green stars, respectively), with output sequences of length $m^1 = 1, m^2 = 1$, and $m^3 = 2$ (denoted with diamonds colored accordingly). The tokens of the initial sequences are shown in panel $(a)$. Panels $(b)$--$(e)$ show the tokens of each sequence after the separation, leader selection, collapse and interpolation steps of the proof. Note that, after the collapse step in panel $(d)$, only tokens selected as leaders are visible. After all steps, sequences match their outputs, as shown in panel $(e)$.}
    \label{fig:thm_sketch}
\end{figure}

\begin{enumerate}
    \item \textbf{Separation.} In this step, we use $\mathcal{O}(N)$ transformer blocks to separate overlapping tokens from different input sequences (see panel $(b)$ in \Cref{fig:thm_sketch}). This step requires alternating feed-forward and self-attention layers.

    \item \textbf{Leader selection.} In this step, we move all tokens to the positive orthant. Then, we select $m^j$ points from each sequence to act as ``leaders'', and position them carefully using $\mathcal{O}(\sum_{j=1}^N m^j)$ transformer blocks (see panel $(c)$ in \Cref{fig:thm_sketch}).

    \item \textbf{Collapse.} In this step, all tokens not selected to be leaders collapse onto the leaders by leveraging the clustering property of a single transformer block (see panel $(d)$ in \Cref{fig:thm_sketch}). 

    \item \textbf{Interpolation.} In this final step, the collapsed sequences are driven sequentially to their outputs (see panel $(e)$ in \Cref{fig:thm_sketch}), which requires at most $\mathcal{O}(\sum_{j=1}^N m^j)$ transformer blocks.
\end{enumerate}
The next subsections make these steps rigorous to prove \Cref{thm:mainResult}. Throughout, we use the crucial fact that the ReLU is zero in a half-space, allowing us to move some tokens while leaving the other ones fixed.

\subsection{Separation}\label{ss:separationStep}

We ensure the $N$ sequences can be made pairwise disjoint through the action of suitably chosen transformer blocks. 
The resulting transformer $\T_{\text{sep}}$ has at most $1 + 2(N-1)$ blocks and is built using the following lemma. In what follows, $B_\delta(x)$ denotes the open ball of radius $\delta > 0$ centered at $x\in \R^d$. 

\begin{lemma}\label{lem:splitOverlappingSequences}
Given pairwise distinct sequences $X^1,\ldots,X^N$, there exists a hardmax transformer $\T_\textnormal{sep}$ of width $d'\leq 4$ and at most $1 + 2(N-1)$ blocks such that
$$
\T_{\textnormal{sep}}({X}^j) \cap \T_\textnormal{sep}({X}^{j'}) = \emptyset \quad \text{for all} \; j\neq j'\in [N].
$$
\end{lemma}
\begin{proof}
We begin by ensuring that all input tokens are mapped to non-zero vectors, which can be easily obtained with a single transformer block globally shifting all tokens (setting $U = 0$, $b=1$ and $W = w \in \R^d$ for any non-zero vector in the FF layer, and $\rho = 1$, $V = 0$, and $A=0$ in the SA layer).
After this initial block, we prove the result constructively by induction on $N$ assuming non-zero tokens, using FF layer constructed as in \Cref{lem:chooseLeader}. Throughout, $x_i^j$ denotes the $i$-th token of sequence $X^j$. 

\noindent $\boldsymbol{N = 2}$. By relabeling tokens if necessary, we may assume without loss of generality that $x_1^2 \in X^2 \setminus X^1$, because the sequences are assumed pairwise distinct. We consider two cases, each of which can be handled using one transformer block.
\begin{itemize}
    \item \textit{Case 1: $x_1^2 \in \ext\co (X^2)$.}
    In this case, the extreme token $x_1^2 \in X^2$ allows for sequence separation using a single transformer block of width $d'=1$. For the $\FF$ layer of this block, we apply \Cref{lem:chooseLeader} with $i^\star = 1$, $j^\star = 2$ to construct a feed-forward layer $\mathbb{FF}$ and a vector $v\in \R^d$ such that, for some $i_1 \in [m^1]$,
    \begin{subequations}\label{eq:lemSeparation}
        \begin{align}
        \C_r (\mathbb{FF} (X^1), v v^\top) = \{ i_1 \} \quad \forall r \in [m^1], \\
        \C_{\ell} (\mathbb{FF} (X^2), v v^\top) = \{ 1 \} \quad \forall \ell \in [m^2].
    \end{align}
    \end{subequations}
    Note that $x_{i_1}^1\in \ext \co (X^1)$ by \Cref{lem:howToLeader}, and $x_{i_1}^1\neq x_1^2$ by construction. For the SA layer, instead, we take $0 < \alpha < 1$ to be specified precisely below and set $\rho = 1 -\alpha$, $V = \alpha I$ and $A=v v^\top$ to obtain
    \begin{align*}
        \SA_{r}(\mathbb{FF} (X^1)) &= {x}_r^1 + \alpha ({x}_{i_1}^1 - {x}_r^1) + \alpha b \frac{v}{\| v \|^2} \quad \forall r \in [m^1] \\
        \SA_{\ell}(\mathbb{FF} (X^2)) &=  {x}_{\ell}^2 + \alpha ({x}_1^2 -{x}_{\ell}^2)  + \alpha b \frac{v}{\| v \|^2} \quad \forall \ell \in [m^2].
    \end{align*}
    (The last term in these two equations comes from the shift introduced by the FF layer.) These quantities are different for all $(r, \ell ) \in [m^1] \times [m^2]$ if and only if
    \begin{equation}\label{eq:lemConditionOnAlpha}
    (1 - \alpha) (x_\ell^2 - x_r^1) + \alpha (x_1^2 - x_{i_1}^1) \neq 0 \quad (r, \ell ) \in [m^1] \times [m^2].  
    \end{equation}
    This condition fails if and only if there exists one pair $( r, \ell )$ and $c_{r \ell} \in \R$ such that $x_\ell^2 - x_r^1 = c_{r \ell} (x_1^2 - x_{i_1}^1)$ and $\alpha$ is the unique solution to $\frac{\alpha}{\alpha - 1} = c_{r \ell}$. This means there are at most finitely many choices of $\alpha$ for which \eqref{eq:lemConditionOnAlpha} fails, so it suffices to pick $\alpha$ not from this set. 
    \begin{figure}
        \centering
        \includegraphics[width=1\linewidth]{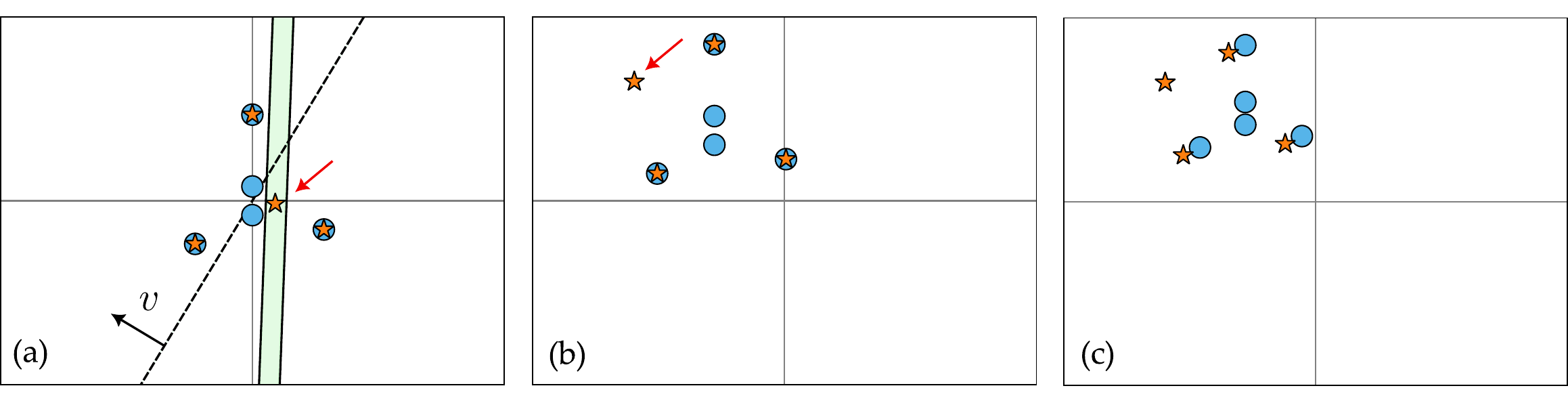}
        \vspace{-3mm}
        \caption{
        Illustration of Case 2 in \Cref{lem:splitOverlappingSequences} for $N=2$ sequences in $\R^2$. The initial tokens in $X^1$ (circles) and $X^2$ (stars) are shown in panel $(a)$, as well as the $\FF$ layer that globally shifts by $v\in \R^2$ and moves the distinct token from $X^2$ (marked with a red arrow). In panel $(b)$, we show the result of applying the resulting FF layer of width $d'=4$. Note how the shifted token is now extreme for $\mathbb{FF}(X^2)$. In panel $(c)$, the sequences have been separated by the $\SA$ layer.}
        \label{fig:case2}
    \end{figure}

    \item \textit{Case 2: $x_1^2 \notin \ext\co (X^2)$.} 
    In this case, we first need to move $x_1^2$ while keeping all other tokens fixed, so it becomes extreme for $\co (X^2)$. We will use a single transformer block of width $d' = 4$.
    
    Fix any $y \notin (\co (X^2) \cup X^1)$. By \Cref{lem:hatFFClassifier}, there exists a width $d' = 3$ feed-forward layer $\Tilde{\FF}$ such that $\Tilde{\FF} (x_1^2) = y$, and $\Tilde{\FF}(x) = x$ for all other tokens $x \neq x_1^2$. Using \Cref{lem:ff-combo-w4} to concatenate $\Tilde{\FF}$ to the width $d'= 1$ feed-forward layer of \Cref{lem:chooseLeader}, we obtain a width $d' = 4$ feed-forward layer $\mathbb{FF}$ such that \Cref{eq:lemSeparation} holds and
    \begin{equation}\label{eq:lem.separation.aux}
        \FF(x_1^2) \in \ext \co \mathbb{FF}(X^2).
    \end{equation}
    This inclusion is true because the extreme points of a set are preserved when shifting the set.
    We can then construct the SA layer in this transformer block following the same construction as in Case 1.
    The action of the overall transformer block is illustrated in \Cref{fig:case2}.
\end{itemize}

\noindent \textbf{Induction step.} Assume there exists a hardmax transformer $\T^{N-1}$ with width $d'\leq 4$ and at most $1 + 2(N-2)$ blocks such that 
$$
    \T^{N-1}(X^j) \cap \T^{N-1}(X^{j'}) = \emptyset \quad \forall j\neq j'\in [N-1].
$$
Set $\Hat{X}^j \coloneqq \T^{N-1}(X^j)$ for all $j$ to ease the notation. Here and throughout, we write $\T_i(X)$ for the $i$-th token of the transformed sequence $\T(X) = \{ \T_1(X), \dots, \T_{\len(X)}(X) \}$. Fix without loss of generality $\Hat{x}_1^N \in \Hat{X}^N$. Our next objective is to build a hardmax transformer $\T'$ with at most two blocks such that
\begin{equation}\label{eq:proofSeparationLemma}
    \T' (\Hat{X}^N) \cap \T' (\Hat{X}^{j}) = \emptyset \quad \text{for all} \; j\in [N-1].
\end{equation}

If $\{ \Hat{x}_1^N \} \cap \Hat{X}^j = \emptyset$ for all $j\in [N - 1]$, then we apply exactly the same argument as in Case 1 or Case 2 above for $N = 2$, depending on whether $\Hat{x}_1^N \in \ext \co (\Hat{X}^N)$ or $\Hat{x}_1^N \notin \ext \co (\Hat{X}^N)$. This already yields \Cref{eq:proofSeparationLemma}.

If $\{ \Hat{x}_1^N \} \cap \Hat{X}^j \neq \emptyset$, instead, we can use the induction hypothesis to obtain (relabeling tokens and the first $N-1$ sequences if necessary) $\{\Hat{x}_1^N \} \cap \Hat{X}^{N-1} = \{ \Hat{x}_1^{N-1} \} \subset \Hat{X}^{N-1}$. Since the sequences are assumed pairwise distinct, one of two possibilities arises.
\begin{enumerate}
    \item[i)] There exists $\Hat{x}_2^N \in \Hat{X}^N$ such that $\Hat{x}_2^N \notin \Hat{X}^{N-1}$. In this case, we can apply the same construction of Case 1 or Case 2 to the token $\Hat{x}_2^N$, yielding \Cref{eq:proofSeparationLemma} with a single-block transformer $\T'$ of width $d' \leq 4$.

    \item[ii)] There exists $\Hat{x}_2^{N-1} \in \Hat{X}^{N-1}$ such that $\Hat{x}_2^{N-1} \notin \Hat{X}^N$. In this case, we first apply Case 1 or Case 2 to $\Hat{x}_2^{N-1}$ to obtain a single-block transformer $\Tilde{\T}$ such that $\Tilde{\T}_1(\Hat{X}^N) \neq \Tilde{\T}_1(\Hat{X}^{N-1})$. Then, apply again Case 1 or Case 2 to the token $\Tilde{\T}_1(\Hat{X}^N)$ to obtain a transformer $\T'$ with two blocks in total and width $d'\leq 4$ satisfying \Cref{eq:proofSeparationLemma}.
\end{enumerate}
Finally, by the induction hypothesis, there exists $\delta_1 > 0$ such that the balls $\{ B_{\delta_1} (\hat{x}_\ell^j )\}_{\ell, j}$ are disjoint for all $\ell \in [\len (X^j)]$ and all $j\in [N-1]$. By taking the parameter $\alpha$ in the blocks of $\T'$ small enough, we can ensure also that the balls $\{ B_{\delta_1} (\T_\ell'(\hat{X}^j)) \}_{\ell, j}$ remain disjoint for all $\ell \in [\len (X^j)]$ and all $j\in [N-1]$, so that $\T' (\Hat{X}^N) \cap \T' (\Hat{X}^{j}) = \emptyset$ for all $j\in [N - 1]$.

In summary, since $\Hat{X}^j = \T^{N-1}(X^j)$ for all $j\in[N]$, we have constructed a transformer $\T_\text{sep} \coloneqq \T' \circ \T^{N-1}$ with at most $1 + 2(N-1)$ blocks and width $d'\leq 4$ achieving $\T_\text{sep}({X}^j) \cap \T_\text{sep}({X}^{j'}) = \emptyset$ for all $j\neq j'\in [N]$. The proof is complete.
\end{proof}

\subsection{Leader selection}\label{ss:leaderSelectionStep}

We slightly abuse notation and set $\Hat{X} \coloneqq \T_{\text{sep}}(X)$. The overall goal of this step is to select $m^j$ tokens in each sequence to act as leaders, and position them in a suitable way such that, in the next step below, we will be able to collapse the remaining $n^j - m^j$ tokens and obtain a sequence with $m^j$ unique tokens, as required by  \Cref{pb:exactInterp}. We proceed as follows.

First, we use a FF layer to globally shift all tokens to $\R^d_{>0}$. Since we have a finite number of tokens, there exists a bounded hypercube
$$
\mathcal{B} = [a_1, a_2]^d \subset \R^d_{>0}
$$
such that setting $\eta = 1$, $W = 0$, $U = u \in \R^d$, and $b = 0$ yields $\FF(x) = x + u \in \mathcal{B}$ for some $0 < a_1 < a_2$.
Next, denote 
$$
\mathcal{P} = \{ x \in \R^d \, : x_i > a_2 \;\; \forall i\in [d]\}.
$$ 
Fix a vector $w\in \R^d_{<0}$ and denote by
$$
Q = \{ \lambda w \; : \lambda \in \R \}
$$
the line spanned by $w$. Moreover, choose distinct points 
$\{x^1, \dots, x^N\} \subset \mathcal{P} \cap Q$ and a hypersphere $S_R$ of arbitrary radius $R > 0$ (see \Cref{fig:leader.selection.sketch} for an illustration of all objects).
\begin{figure}
    \centering
    \includegraphics[width=0.9\linewidth]{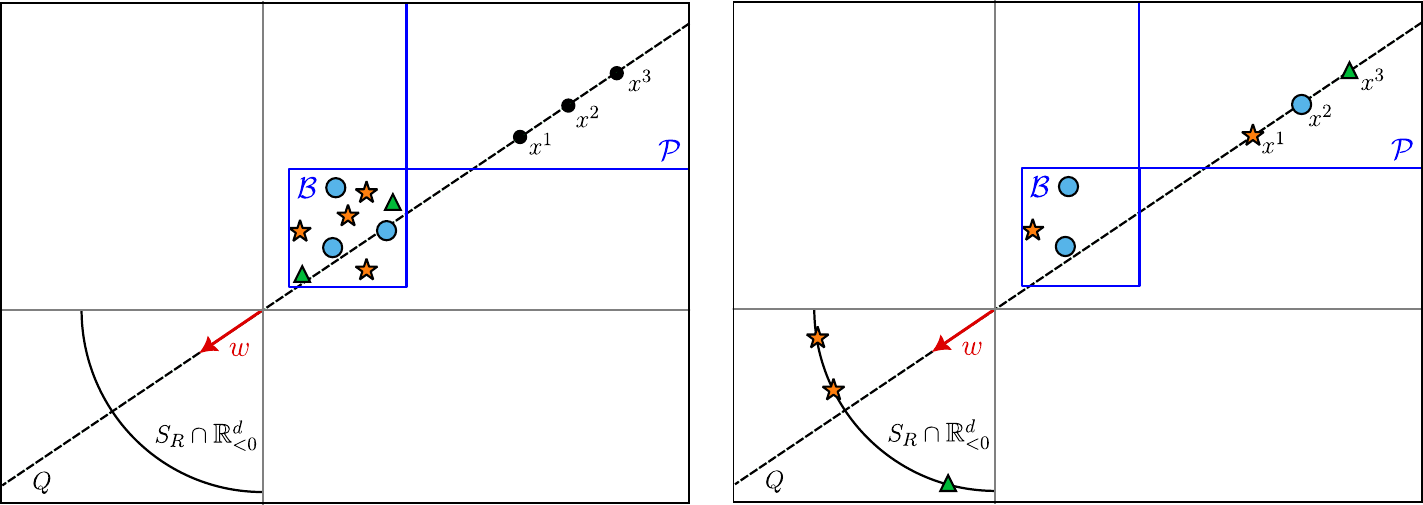}
    \caption{Sketch of the geometric construction in $\R^2$ used in the leader selection step of the proof of \Cref{thm:mainResult}. In the left panel, all tokens from the $N = 3$ sequences $X^1$ (stars), $X^2$ (circles) and $X^3$ (triangles) lie in the hypercube $\mathcal{B}$. In the right panel, after the leader selection step, $m^1 = 3$, $m^2 = 1$ and $m^3 = 2$ leaders have been selected for each sequence, respectively. For each sequence, one leader is placed in $\mathcal{P}$ and the remaining ones (if any) are placed on $S_R \cap \R^d_{<0}$.}
    \label{fig:leader.selection.sketch}
\end{figure}
For every $j\in [N]$, we use \Cref{lem:hatFFClassifier} iteratively to obtain:
\begin{itemize}
    \item $N$ feed-forward layers that map $\hat{x}_1^j \in \hat{X}^j$ to $x^j$ for each $j\in [N]$ in turn, while keeping all other tokens fixed.
    \item $\sum_{j = 1}^{N} (m^j -1)$ feed-forward layers that map $\hat{x}_2^j, \dots, \hat{x}_{m^j}^j \in \hat{X}^j$ to distinct points in $S_R \cap \R^d_{<0}$ for all $j\in [N]$ in turn, while keeping all other tokens fixed and ensuring sequences remain disjoint.
\end{itemize}
The transformer $\T_{\text{lead}}$ implementing this construction has width $d'\leq 3$ and $1 + \sum_{j=1}^N m^j$ blocks in which the self-attention layers act as the identity ($\rho = 1$, $V = 0$, $A = 0$).

\subsection{Collapse}\label{ss:collapseStep}

In this step, we use the clustering properties of the SA layer to collapse all tokens in the box $\mathcal{B}$ to $x^1, \dots, x^N$, while keeping the rest fixed. This is obtained by \Cref{lem:partialConc} with a single-block transformer $\T_{\text{col}}$ where the FF layer acts as the identity ($\eta = 1$, $W = 0$, $U = 0$, $b = 0$) and the parameters in the SA layer are set to $\rho = 0$, $V = I_d$, and $A = I_d$. We thus conclude that $(\T_\text{col} \circ \T_\text{lead} \circ \T_\text{sep}) (X^j)$ has $m^j$ unique elements for all $j\in [N]$.

\subsection{Interpolation}\label{ss:interpolationStep}

The transformer $\T_\text{col} \circ \T_\text{lead} \circ \T_\text{sep}$ constructed so far has collapsed the $n^j$ distinct tokens of each input sequence $X^j$ into $m^j$ distinct tokens, as required by the output sequence $Y^j$. The remaining task is to move each token iteratively to its desired output. For every $j \in [N]$, we use \Cref{lem:hatFFClassifier} to obtain $m^j$ FF layers that sequentially map the $m^j$ unique tokens of the collapsed sequence $(\T_\text{col} \circ \T_\text{lead} \circ \T_\text{sep}) (X^j)$ to the $m^j$ outputs $y_1^j, \dots, y_{m^j}^j$. This can be implemented by a transformer $\T_{\text{int}}$ with at most $\sum_{j=1}^N m^j$ blocks, where each block contains one FF layer provided by \Cref{lem:hatFFClassifier} together with a SA layer acting as the identity ($\rho = 1$, $V = 0$, $A = 0$). 

\subsection{Overall construction and complexity estimate}\label{ss:complexityEstimate}

Combining the transformers described in the previous steps results in a transformer satisfying 
$$
(\T_{\text{int}} \circ \T_{\text{col}}\circ \T_{\text{lead}} \circ \T_{\text{sep}})(X^j) = Y^j
$$
for all $j\in [N]$, as desired.

Next, we count the number of blocks in our transformer. The maximum width $d' = 4$ was required in the separation step, and at most 
$$
\underbrace{1 + 2(N-1)\vphantom{\sum_{j=1}^N}}_{\T_\text{sep}}
+ \underbrace{1 + \sum_{j=1}^N m^j}_{\T_\text{lead}}
+ \underbrace{1\vphantom{\sum_{j=1}^N}}_{\T_\text{col}}
+ \underbrace{\sum_{j=1}^N m^j}_{\T_\text{int}}
= 2\sum_{j=1}^N m^j + 2N + 1
$$
blocks were needed in total. 
More explicitly, ignoring layers that act as the identity, the construction is a sequence of applications of the key technical lemmas from \Cref{sec:technicalResults}. 
The separation step uses one FF layer to shift all tokens, followed by $2(N-1)$ transformer blocks constructed with \Cref{lem:splitOverlappingSequences} to make the $N$ sequences pairwise disjoint. 
In the leader selection step, we shift all tokens to the positive orthant using one FF layer, then apply $\sum_{j=1}^N m^j$ FF layers of the type provided by \Cref{lem:hatFFClassifier} to position the leaders. 
The collapse step applies one self-attention layer of the type constructed in \Cref{lem:partialConc} to collapse tokens onto the leaders. 
Lastly, in the interpolation step we apply $\sum_{j=1}^N m^j$ FF layers provided by \Cref{lem:hatFFClassifier} to move the collapsed tokens to the desired outputs.

Finally, the parameters in each block have been chosen as vectors in $\R^d$, rank-one matrices, multiples of the identity, or constants, resulting in $\mathcal{O}(d)$ parameters per block. The total number of parameters is therefore $\mathcal{O}(d \sum_{j=1}^N m^j)$.

\section{Proof of \texorpdfstring{\Cref{thm:mainResultSoftmax}}{Theorem \ref{thm:mainResultSoftmax}}}\label{sec:proofOfMainResultSoftmax}

We adapt the steps in \Cref{sec:proofOfMainResult} to prove \Cref{thm:mainResultSoftmax}. As discussed in \Cref{ss:mainResults}, the result is not simply obtained by continuous perturbation arguments, as these would only yield an approximate interpolation result. Indeed, a key limitation of the softmax self-attention formulation in \Cref{eq:softmaxFormulation} is that, for non-zero $A\in \R^{d\times d}$, even without residual ($\rho = 0$) it always yields a strict combination of tokens, not allowing for exact collapse\footnote{Setting $A = 0$ and $\rho = 0$, all tokens collapse exactly to their average, fundamentally limiting the output sequences length to $1$ (we point, however, to \cite{geshkovski2024measure} for a workaround using $\rho = 1$).}. Therefore, we adapt the collapse step to include an extra transformer with $\mathcal{O}(N)$ blocks that utilizes non-residual feed-forward layers with ReLU activation to iteratively collapse small neighborhoods of tokens.

In what follows, we make the previous discussion rigorous, devoting a section to each step of the proof that needs to be adapted; namely, the separation, leader selection, and collapse steps. 

\subsection{Softmax separation}

The technical construction in \Cref{lem:splitOverlappingSequences} needs to be adapted to ensure that there exists $\tau_0 > 0$ such that $\T^\tau_{\text{sep}}(X^j) \cap \T^\tau_{\text{sep}}(X^{j'}) = \emptyset$ for all $j\neq j'\in [N]$ and all $\tau \leq \tau_0$.

We will solely invoke continuity arguments for this part, and restrict the presentation to a single sequence $X$ ($N =1$) and a single self-attention layer. Since we are working with a finite set of sequences, the continuity arguments are easily extended to general $N \in \N$.

Fix a finite set $X = \{ x_i \}_{i\in [n]}$ and $A \in \R^{d\times d}$. Denote for every $i\in [n]$
\begin{align*}
    \mathcal{M}_i^0 (X, A) &\coloneqq \frac{1}{|\mathcal{C}_i(X, A)|} \sum_{j\in \mathcal{C}_i(X, A)}x_j, \\
    \mathcal{M}_i^\tau (X, A)  &\coloneqq  \sum_{j}\frac{\exp{\left(\frac{1}{\tau}\langle A x_i, x_j \rangle\right)}}{\sum_{k} \exp{\left(\frac{1}{\tau}\langle A x_i, x_k \rangle\right)}} x_j.
\end{align*}
As discussed in \Cref{sss:SAlayers}, for fixed $X$ and $A$, the softmax function $\pi^\tau$ and the hardmax function $\pi^0$ satisfy \cite{elfadelSoftmax1993}
$$
\lim_{\tau \to 0} \pi_{i\ell}^\tau (X, A) = \pi_{i \ell}^0 (X, A) \quad \text{for all} \; i,\ell \in [n].
$$
Equivalently, for every $\delta > 0$ there exists $\tau_0(\delta) > 0$ such that
$$
    \| \mathcal{M}_i^\tau (X, A) -\mathcal{M}_i^0 (X, A)  \| < \delta \quad \forall \tau \leq \tau_0(\delta)
$$
This means that, when comparing softmax and hardmax SA layers,
\begin{align*}
    \SA^{\tau}_i(X) &= \rho x_i + \alpha \mathcal{M}^\tau_i (X,A), \\
    \SA^{0}_i(X) &= \rho x_i + \alpha \mathcal{M}^0_i(X, A),
\end{align*}
we obtain the estimate
\begin{equation}\label{eq:sm-hm-estimate}
    \|\SA^{0}_i(X) - \SA^{\tau}_i(X) \| = | \alpha | \| \mathcal{M}_i^0(X,A) - \mathcal{M}_i^\tau(X,A) \| < |\alpha| \delta
\end{equation}
for all sufficiently small $\tau$. Moreover, every hardmax SA layer in the proof of \Cref{lem:splitOverlappingSequences} ensures that tokens remain pairwise distinct. In particular, for some $\varepsilon > 0$ and all $i\neq i'\in [n]$ we have
$$
    \| \SA^{0}_i(X) - \SA^{0}_{i'}(X) \| > \varepsilon.
$$
Using \Cref{eq:sm-hm-estimate}, we can similarly obtain for softmax SA layers that satisfy
\begin{align}\nonumber
    &\| \SA^{\tau}_i(X) - \SA^{\tau}_{i'}(X) \| \\ \nonumber
    &\geq \| \SA^{0}_i(X) - \SA^{0}_{i'}(X) \| - \| \SA^{\tau}_i(X) - \SA^{0}_i(X) \| - \|\SA^{0}_{i'}(X) - \SA^{\tau}_{i'}(X) \| \\ \nonumber
    &\geq \varepsilon - 2 |\alpha|\delta.
\end{align}
Therefore, choosing $\delta < \frac{\varepsilon}{2 |\alpha|}$, for $\tau$ small enough in each softmax SA layer, we obtain that tokens remain distinct with the same construction. 

\subsection{Softmax leader selection}

We proceed exactly as in the leader selection step in the proof of \Cref{thm:mainResult}, with the extra requirement that there is ``enough separation'' between the desired positions of the $m^j$ tokens in each sequence that will become leaders in the next step.  Denote $\hat{X} \coloneqq \T_\text{sep}^\tau (X)$. For a given $\delta > 0$, we choose $w \in \R^d$, $\| w \| = 1$, and $R > 0$ so the desired positions for the leaders in the next step $\{ q_i^j \}_{i\in [m^j], j\in [N]}$ satisfy
\begin{equation}\label{eq:deltaSepProjections}
    \left\| w \left( \ip{w}{q_i^j - q_{i'}^{j'}}\right) \right\| \geq 2 \delta
\end{equation}
for all $i\neq i'$ and all $j\neq j'$. Indeed, this is clearly possible for those points $q_i^j \in \R_{>0}^d \cap Q$. Additionally, for every distinct pair of points $q_i^j, q_{i'}^{j'} \in \R_{<0}^d \cap S_R$, we have
\begin{equation}\label{eq:softmaxLeaderSelect}
\left\| w \left( \ip{w}{q_i^j - q_{i'}^{j'}}\right) \right\| = R \left| \ip{w}{\hat{q}_i^j  - \hat{q}_{i'}^{j'}}\right|,
\end{equation}
where $\hat{q}_i^j = {q_i^j} / R$ and $\hat{q}_{i'}^{j'} = {q_{i'}^{j'}} / R$. We first ensure that the right-hand side of \eqref{eq:softmaxLeaderSelect} is non-zero by choosing $w\in\R^d$ outside a finite set of directions. Then, we fix 
$$
\zeta = \min_{i\neq i', j\neq j'} \left| \ip{w}{\hat{q}_i^j - \hat{q}_{i'}^{j'}}\right|
$$
and choose $R > {2\delta} / {\zeta}$ to satisfy \Cref{eq:deltaSepProjections}.
\begin{figure}
    \centering
    \includegraphics[width = 0.9\textwidth]{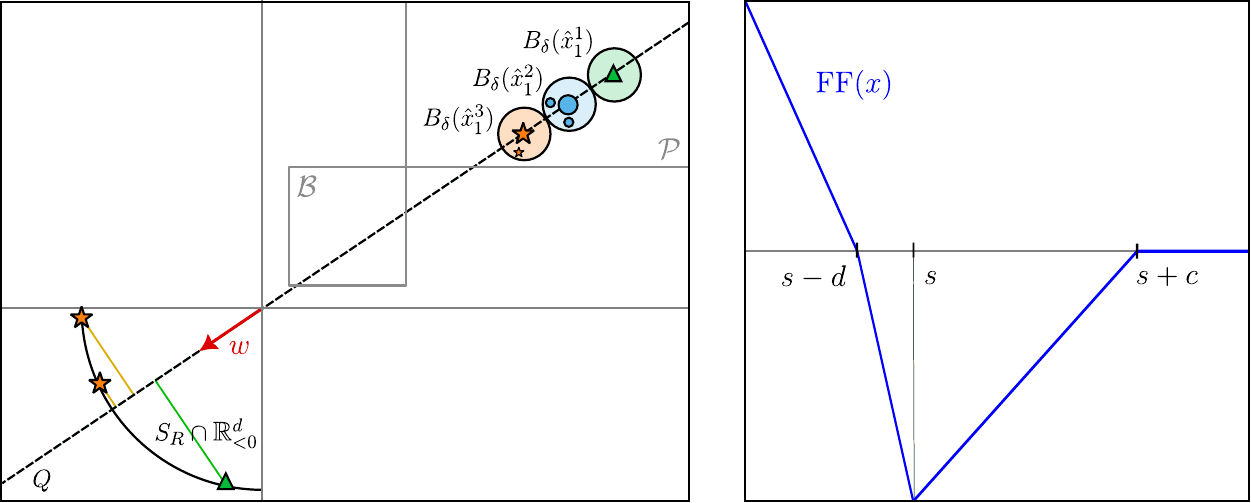}
    \caption{Geometric construction of the softmax collapse step for the proof of \Cref{thm:mainResultSoftmax}. In the left panel, the result of the first substep, with all tokens in a ball of radius $\delta$ of the leaders. In the right panel, the $\FF$ layer used in the second substep.}
    \label{fig:FFcollapseStep}
\end{figure}

\subsection{Softmax collapse} 

The collapse step in the softmax setting requires additional care, since it must be carried out in two substeps. We abuse notation and denote $\hat{X}^j \coloneqq (\T_\text{lead}^\tau \circ \T_\text{sep}^\tau) (X^j)$.

In the first substep, for sufficiently small $\tau > 0$ and thanks to \Cref{eq:deltaSepProjections}, a softmax SA layer with $\rho = 0$, $V = I_d$ and $A = I_d$ moves all tokens of each sequence to a ball of radius $\delta$ of the $m^j$ leaders of that sequence, arranged in the previous leader selection step. The resulting configuration is shown in the left panel of \Cref{fig:FFcollapseStep} for a low-dimensional example.

In the second substep, the $N$ balls $B_\delta (\hat{x}_1^1), \dots,B_\delta (\hat{x}_1^N) \in \R^d_{>0}$ are collapsed sequentially, using non-residual feed-forward layers with ReLU activation as follows. Recall that, by construction, we have that the $N$ balls are disjoint and so are their projections onto the line 
$$
Q = \{ \lambda w \; : \lambda \in \R \}.
$$ 
We relabel the tokens $\hat{x}_1^1, \dots, \hat{x}_1^N$ if necessary, so that they are ordered from farthest ($\hat{x}_1^1$) to closest ($\hat{x}_1^N$) to the origin. Now, fix a non-residual ($\eta = 0$) FF layer of width $d'=3$ as follows (see right panel of \Cref{fig:FFcollapseStep}):
$$
\FF(x) =
\begin{bmatrix}
\vert & \vert & \vert \\
-w & \left(\tfrac{c}{d} + 1\right) w & -\tfrac{c}{2d} w \\
\vert & \vert & \vert
\end{bmatrix}
\sigma \left(
\begin{bmatrix}
\rule{0.4cm}{0.4pt} & - w^\top & \rule{0.4cm}{0.4pt} \\
\rule{0.4cm}{0.4pt} & - w^\top & \rule{0.4cm}{0.4pt} \\
\rule{0.4cm}{0.4pt} & - w^\top & \rule{0.4cm}{0.4pt} \\
\end{bmatrix} x +
\begin{bmatrix}
s + c \\
s \\
s - d
\end{bmatrix}
\right).
$$
Recall that $w\in \R^d_{<0}$ is fixed in the leader selection step. Now, choose $s\in \R$ and $c, d > 0$ such that the interval $[s - d, s + c] \subset \R$ contains only the projections of the balls $B_\delta (\hat{x}_1^{j})$ onto $Q$ for all $j > 1$. Then, after applying the previously described $\FF$ layer:
\begin{itemize}
    \item the ball $B_\delta (\hat{x}_1^1)$ collapses to $0$,
    \item the balls $B_\delta (\hat{x}_1^j)$ are mapped to disjoint segments $I^j \subset \R^d_{>0} \cap Q$ for all $j > 1$,
    \item the remaining $m^j - 1$ leaders per sequence $\hat{x}_i^j \in \R^d_{<0}$, $i\neq 1$, are projected to distinct (as ensured by \Cref{eq:deltaSepProjections}) positive multiples of $w$, thus remaining in $\R^d_{<0}$.
\end{itemize}
This process can now be iterated for the next ball of radius $\delta$ (now collapsed to a segment of $\R_{>0}^d \cap Q$) furthest away from the origin, which corresponds to $I^N$ because the order has been flipped. At each iteration, the ball collapsed to the origin in the previous step is shifted along $w$, preventing overlaps with subsequent segments.

After $N$ iterations, we obtain a transformer $\T^\tau_{\text{col}}$ of width $d'=3$ and $N$ blocks, whose SA layers act as the identity ($\rho=1$, $V=0$, $A=0$). This transformer collapses each input sequence to exactly $m^j$ distinct tokens, after which the interpolation construction of \Cref{thm:mainResult} applies verbatim.

\begin{remark}\label{rem:HMvsSMconstruction}
    Our construction only requires that, within each self-attention layer, $\tau$ be sufficiently small to achieve the desired behavior. This does not imply that subsequent layers produce identical outputs across models, as illustrated in \Cref{fig:softmaxRemark}:
    \begin{itemize}
        \item In the hardmax transformer, the first block of the separation step may produce a token configuration that necessitates another separation step.
        \item In the softmax transformer, the same block may instead lead directly to a configuration suitable for the collapse step.
    \end{itemize}
    Thus, while intermediate token dynamics may differ between hardmax and softmax transformers, the overall construction remains valid and our results hold in both cases.
\end{remark}
\begin{figure}
    \centering
    \includegraphics[width=0.6\linewidth]{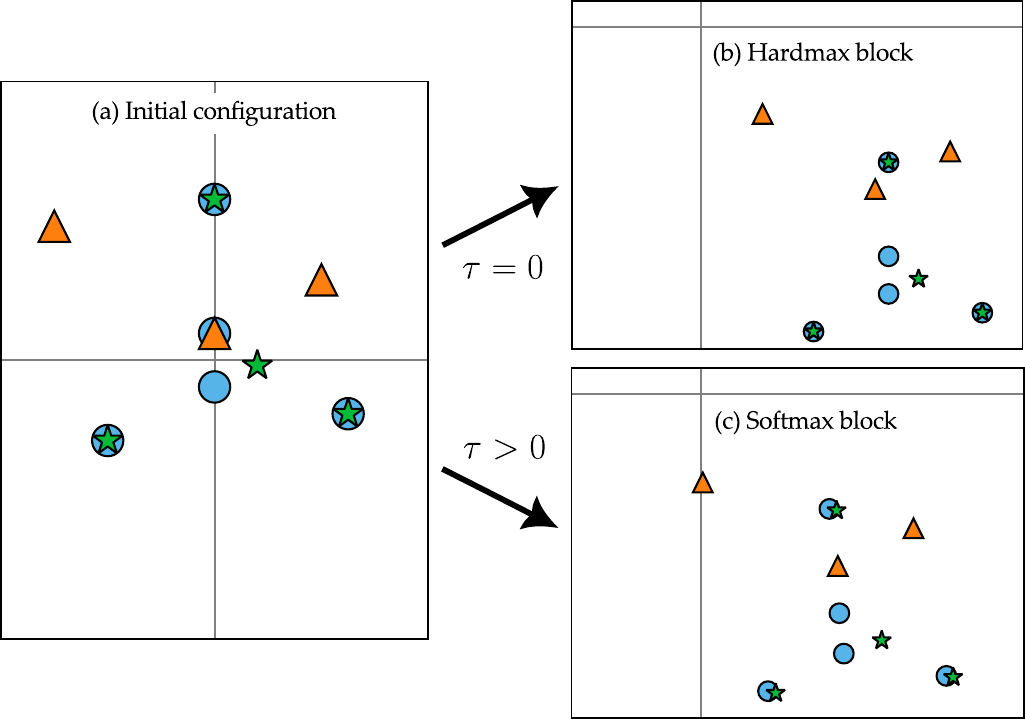}
    \caption{Illustration of \Cref{rem:HMvsSMconstruction} for $N=3$ sequences in $\R^2$. In panel $(a)$, the initial sequences' configuration, with tokens from different sequences denoted with different colors and shapes. Panel $(b)$ shows the effect of one hardmax separation block, while panel $(c)$ shows the effect of a softmax block with $\tau = 1$. Unlike the hardmax case, the softmax block can already separate all sequences, allowing the construction to proceed directly to the collapse step.}
    \label{fig:softmaxRemark}
\end{figure}

\section{Conclusion}\label{sec:conclusion}

We summarize here the main results of the paper and outline several directions for future work.

\subsection{Summary of results}

We presented the first exact interpolation result for transformers in sequence-to-sequence tasks that applies to deep transformers with standard self-attention layers and to real vector-valued outputs. Specifically, we constructed an explicit deep transformer that exactly interpolates real-valued sequences with a number of parameters independent of the input sequence length. Our construction mirrors key features of practical transformers, including the alternation of self-attention and feed-forward layers, and the low-rank structure of self-attention.

Our theoretical framework reveals two fundamental roles of self-attention layers. First, they act as a dimensionality reduction mechanism through a clustering effect, enabling efficient interpolation. Second, they leverage the non-locality of self-attention to disentangle tokens shared across different sequences. These insights help explain the practical efficiency of transformers in sequence interpolation tasks and highlight their parameter efficiency compared to traditional architectures like ResNets. 
In particular, we have shown that self-attention allows for a significant reduction in the number of parameters: our transformer requires $\mathcal{O}(d \sum_{j=1}^N m^j)$ parameters, while a comparable ResNet architecture requires $\mathcal{O}\left(d N n \right)$ parameters, where $n$ denotes a fixed input length (e.g., $n = \max_j n^j$). Since in most practical settings where transformers excel, such as masked language modeling, \( m^j \ll n \), this translates into a reduction in parameter count, especially for long input sequences. Furthermore, the number of parameters in our transformer is independent of the input sequence lengths $n^j$ themselves, in contrast to ResNets.

\subsection{Perspectives and open questions}
Our work leaves room for several future improvements, which we discuss next.

\subsubsection{Complexity, parameter norms and separation}\label{sss:paramNorms}

While our transformer construction exhibits lower complexity, measured in terms of the number of blocks and parameters, than those proposed in previous works \cite{geshkovski2024measure}, we do not know whether it is optimal. In practice, transformers often achieve excellent \textit{approximate} interpolation performance with a number of blocks significantly smaller than the dataset size. This raises the question of whether exact interpolation could also be attained with lower complexity. Addressing this question remains an interesting open problem.

Beyond block and parameter counts, an important open direction is to understand the magnitude of the parameters required in our construction. A priori bounds on the norm of the parameters would be relevant, for example, in the context of regularized training as pointed out in \Cref{sec:implicationsTraining}. Intuitively, one expects the required norms to depend on how well-separated the tokens are in each step of our construction: well-separated configurations may be updated as desired using transformer blocks with moderate parameter values, whereas tightly clustered data may necessitate larger parameters. 
An illustrative example arises in the FF layers acting as ``hat functions'' of \Cref{lem:hatFFClassifier}, where the parameters need to blow up as data separation tends to zero. A systematic study of this effect, both theoretically and numerically, would provide valuable insight into exact interpolation constructions and is left as further work.

\subsubsection{Extension to probability vector outputs}\label{sss:extensionProbVectorOutputs}

An interesting question is related to next-token prediction tasks in language applications. In such settings, it is common that the output sequence is not a single token but a probability vector over the vocabulary (that is, the total number of distinct words in all sequences). One interesting question is then whether transformers can exactly match a finite set of probability vectors. Our construction also applies to this setting under suitable modifications. Specifically, let $\mathcal{V}\in \N$ be the total number of elements in the vocabulary. Then, if the input sequence is given by $X \in (\R^{\mathcal{V}})^n$, the goal of the model is to output the probability vector
$$
Y = (p_1, \dots, p_\mathcal{V}) \in \R^\mathcal{V}, \quad \text{for} \; p_i > 0.
$$
One can easily obtain an exactly interpolating transformer $\T$ which outputs $Y$, provided $d = \mathcal{V}$. This probability-vector setting is closely related to the next-token prediction capacity studied in \cite{madden2025next}, where model complexity bounds are established for one-layer multi-head decoder-only transformers. The observation above shows that exact matching of finitely many probability vectors is also achievable within our construction when $d=\mathcal{V}$. In this sense, our result would be complementary to the results in \cite{madden2025next}, as it arises from a deep single-head sequence-to-sequence architecture. In practice, however, the embedding dimension $d$ is typically much smaller than $\mathcal{V}$. Extending our framework to this low-dimensional regime remains an interesting direction for future work.

\subsubsection{Multi-head attention}\label{sss:multihead}

Another relevant question is whether our arguments can be efficiently extended to transformers with multi-head attention. As heads can be seen as an analogue of width in ResNets, where a trade-off with depth has been shown \cite{alvarez2024interplay}, we wonder if multiple heads can reduce the number of parameters and blocks required to achieve perfect sequence interpolation. For this, one should replace \Cref{eq:selfatt_a} with
$$
    \mathrm{SA}_i(X) = \rho x_i + \sum_{h=1}^H V^h \sum_{\ell = 1}^{\len(X)} \pi_{i\ell}(X,A^h)x_\ell,
$$
for a given number of heads $H\in \N$ and head-specific parameters $A^h, V^h \in \R^{d\times d}$, $h\in [H]$. Our current analysis rigorously covers the case of a single head $H = 1$ and justifies the good performance of these models given sufficient depth \cite{michel2019sixteen}. It would be interesting to investigate whether the additional degrees of freedom provided by multiple heads can reduce the required depth of the transformer, in particular the number of blocks used in the separation step described in \Cref{ss:separationStep}.

\subsubsection{Masked self-attention}\label{sss:maskedSA}

Masking is a common feature of self-attention layers for next-token interpolation tasks in natural language processing, where a token $x_i$ should only be influenced by those preceding it. To model this, one should replace \Cref{eq:selfatt_a} with
$$
    \mathrm{SA}_i(X) = \rho x_i + V \sum_{\ell = 1}^i \pi_{i\ell}(X,A)x_\ell.
$$
Notice that, for each $i\in [\len (X)]$, the sum runs only over $\ell \leq i$. This masked self-attention setting can be handled by our results in \Cref{thm:mainResult} and \Cref{thm:mainResultSoftmax} through a problem reformulation as follows (we assume that $n^j = n$ and $m^j = m$ to simplify the exposition). For all $j\in [N]$, 
\begin{itemize}
    \item each $X^j$ is split into $n$ initial sequences $ X_t^j = \{ x_1^j, \dots, x_t^j \} \subset (\R^d)^t$ for all $t\in [n]$,
    \item each $Y^j$ is split into $n$ output sequences $ Y_t^j = \{ y_t^j \} \subset \R^d$ for all $t\in [m]$ and $Y_t^j = \{ y_m^j\}$ for all $t > m$.
\end{itemize}
Now, we can apply \Cref{thm:mainResult} or \Cref{thm:mainResultSoftmax} to the dataset $\{(X_t^j, Y_t^j)\}_{j\in [N], t\in [n]}$ consisting of $N n$ sequences, and obtain a transformer $\T$ such that $\T (X^j_t) = Y^j_t$ with $\mathcal{O}(Nn)$ blocks and $\mathcal{O}(d N n)$ parameters. Thus, it holds for a causal transformer $\T_\text{causal}$ that
$$
\T_\text{causal} (X^j) = \{ \T(X_1^j),  \dots, \T(X_n^j) \} = \{ y_1^j,  \dots,  y_m^j \} = Y^j,
$$
as desired. This reformulation leads to a complexity that scales with input length $n$ rather than output length $m$, a scaling also observed in large-scale models~\cite{hoffmann2022training}. Designing architectures that avoid this dependence while preserving exact interpolation remains an open challenge.






\bmhead{Acknowledgements}
The authors thank Martín Hernández, Domènec Ruiz-Balet and Borjan Geshkovski for valuable conversations.



\section*{Declarations}


\begin{itemize}
\item Funding

The work of AA was funded by the European Union's Horizon Europe MSCA project ModConFlex (grant number 101073558). 
The work of GF is funded by the ANR-DFG project MONET (DFG project number 568735456).
The work of EZ was funded by the Alexander von Humboldt-Professorship program, the ERC Advanced Grant CoDeFeL, the Grants PID2020-112617GB-C22 KiLearn and TED2021-131390B-I00-DasEl of MINECO and PID2023-146872OB-I00-DyCMaMod of MICIU (Spain),  the European Union's Horizon Europe MSCA project ModConFlex (grant number 101073558), the Transregio 154 Project ``Mathematical Modelling, Simulation and Optimization Using the Example of Gas Networks'' of the DFG, the AFOSR 24IOE027 project, the Madrid Government--UAM Agreement for the Excellence of the University Research Staff in the context of the V PRICIT (Regional Programme of Research and Technological Innovation), and the SURE-AI Centre grant 357482, Research Council of Norway. 

\item Competing interests 

The authors declare no competing interests.

\item Ethics approval and consent to participate

Not applicable

\item Consent for publication

All authors consent to the publication of this manuscript.

\item Data availability 

Not applicable

\item Materials availability

Not applicable

\item Code availability 

Code available in the GitHub repository: 

\url{https://github.com/DCN-FAU-AvH/exactInterpTF}

\item Author contribution

All authors contributed to the study conception and design. The initial idea was developed by A. Alcalde and E. Zuazua. The analysis was mainly performed by A. Alcalde and G. Fantuzzi. The manuscript was first written by A. Alcalde and subsequently reviewed and edited by all authors. All authors read and approved the final manuscript.
\end{itemize}

\bibliography{sn-bibliography}

\end{document}